\documentclass[a4paper,12pt]{article}

\usepackage{amsmath,amssymb,amsfonts,amsthm,natbib}
\usepackage{vmargin}
\setpapersize{A4}

\usepackage[colorlinks=true,linkcolor=blue,citecolor=blue]{hyperref}

\newcommand{\e}{\mathrm{e}}
\newcommand{\inner}[2]{\left\langle #1, #2 \right\rangle}
\newcommand{\g}[1]{\boldsymbol{#1}}

\newcommand{\R}[0]{\mathbb{R}} 
\newcommand{\E}[0]{\mathbb{E}}

\newcommand{\I}[1]{\boldsymbol{1}_{#1}}
\renewcommand{\H}[0]{\mathcal{H}} 
\newcommand{\X}[0]{\mathcal{X}} 
\newcommand{\Y}[0]{\mathcal{Y}} 

\newcommand{\F}[0]{\mathcal{F}} 
\newcommand{\G}[0]{\mathcal{G}} 
\renewcommand{\L}[0]{\mathcal{L}} 
\newcommand{\Z}[0]{\mathcal{Z}}

\newcommand{\A}[0]{\mathcal{A}} 

\newcommand{\sign}[0]{\mbox{sign}} 
\newtheorem{theorem}{Theorem}
\newtheorem{corollary}{Corollary}
\newtheorem{lemma}{Lemma}

\newtheorem{definition}{Definition}

\newtheorem{example}{Example}

\newcommand{\argmin}{\operatornamewithlimits{argmin}}

\begin{document}

\title{\bf Uniform Risk Bounds for Learning\\with Dependent Data Sequences}

\author{Fabien Lauer\bigskip\\LORIA, Université de Lorraine, CNRS, France}
\date{}

\maketitle

\abstract{This paper extends standard results from learning theory with independent data to sequences of dependent data. Contrary to most of the literature, we do not rely on mixing arguments or sequential measures of complexity and  derive uniform risk bounds with classical proof patterns and capacity measures. In particular, we show that the standard classification risk bounds based on the VC-dimension hold in the exact same form for dependent data, and further provide Rademacher complexity-based bounds, that remain unchanged compared to the standard results for the identically and independently distributed case.  
Finally, we show how to apply these results in the context of scenario-based optimization in order to compute the sample complexity of random programs with dependent constraints.  
}

\section{Introduction}
\label{sec:intro}

Statistical learning theory offers probabilistic guarantees on the accuracy of models learned from data. Most of these results assume that the data come from a realization of a sample of independent and identically distributed (i.i.d.) random variables, which allows one to build the theory upon standard concentration arguments. 
However, this assumption is often unrealistic as dependent data are ubiquitous in real-world applications, such as signal processing, speech recognition, biological sequence annotation \citep{Baldi01}, dynamical system identification \citep{Ljung87}, or even handwritten character recognition where the images collected for training come from a string of letters forming a meaningful text. 

This paper extends several classical results to sequences of dependent data, such as risk bounds based on the Vapnik-Chervonenkis (VC) dimension or the Rademacher complexity. In particular, we focus on {\em uniform} risk bounds that are more suitable for nonconvex loss functions difficult to minimize in practice. 

As a motivating application, we also consider the consequences of these results in the framework of scenario-based optimization for solving uncertain optimization problems. Here, robust solutions are those that typically satisfy an infinite number of constraints: one for each value of the uncertain parameter of the problem. Scenario-based optimization computes instead probably approximately correct solutions by sampling the set of uncertainties and solving the problem with a finite number of constraints. In this context, the computing complexity is directly related to the number of sampled constraints and it is thus of primary importance to compute the sample complexity of the corresponding random program, i.e., the smallest sample size for which we can guarantee with a high confidence that the probability of violation of the constraint is low. 

\paragraph{Related work} 
Dependence between training instances typically occur for machine learning in the context of ranking problems, where the training algorithms deal with overlapping pairs of input data. In this setting, prior work on generalization bounds \citep{Usunier06,Ralaivola15} relied on a decomposition of the training sample into independent subsamples for which concentration could be applied by following the graph-coloring scheme of \cite{Jansen04}. 

The literature also contains numerous results that do not assume that such a decomposition is possible, e.g., when a single training instance depends on all the others as is common when handling time-series or sequential data. In that case, the general approach is to measure the degree of dependence between the data points with a mixing coefficient \citep{Bradley05} and assume that this coefficient tends to zero sufficiently quickly to allow for the derivation of meaningful bounds. Works in this line include that of \cite{Meir00,Steinwart09,Mohri09,Mohri10} and rely on technical arguments inherited from \cite{Yu94}. Though the obtained risk bounds share most of their structure with their counterpart for i.i.d. data, they also involve the mixing coefficient, which slightly degrades the convergence and remains difficult to determine or estimate in practice. Note that a connection between the mixing and graph-coloring arguments is discussed in \cite{Ralaivola10}.  

Other approaches based on Rademacher complexities include that of \cite{Rahklin15}, which can also deal with non-stationary sequences \citep{Kuznetsov15}, but involves complex computations with tree processes and decoupling techniques from \cite{delaPena99}. Also, these works consider a different form of the risk (a conditional forecasting risk) and rely on sequential counterparts of standard capacity measures.   

A more recent line of research developped by \cite{Simchowitz18,Faradonbeh18} relies on techniques from \cite{Mendelson14,Mendelson18} to bypass the need for mixing arguments. However, these results apply only to the empirical risk minimizer (the orthogonal least-squares estimator is considered in \cite{Simchowitz18,Faradonbeh18}). Thus, they do not provide uniform risk bounds that apply to any model in a given class, which is critical for applications with nonconvex loss functions where the empirical risk minimizer remains elusive (such as unsupervised learning or hybrid dynamical system identification \citep{Lauer19}).

Regarding our motivating application, i.e., scenario-based optimization, we can distinguish two main lines of research. The first, developed for instance in \cite{Alamo09,Lauer23}, builds upon learning theory to derive bounds on the probability of violation and sample complexities. The second, pioneered by \cite{Campi08,Calafiore10}, relies instead on convex analysis arguments that lead to tighter bounds for the specific case of convex optimization. A number of works also tried to extend the latter to various forms of nonconvex programs \citep{Esfahani14,Campi18}. However, all these works, either based on learning theory or convex analysis, assume that the scenarios are sampled independently, which could be problematic for certain applications, as, e.g., the one of \cite{Wang21}.

\paragraph{Contribution}
We show that many standard results from learning theory, such as classification VC bounds and Rademacher complexity-based bounds, apply {\em without any modification} to dependent data. In addition, we derive these results with proofs that follow the standard patterns and rely on the classical capacity measures, which stands in contrast to the approach of \cite{Rahklin15} which requires additional arguments and more complex computations with sequential complexities. 
In comparison with other works from the literature, we obtain {\em uniform} risk bounds that do not introduce additional terms or mixing coefficients and that are thus more widely applicable than those of \cite{Simchowitz18,Faradonbeh18} and tighter than those of, e.g., \cite{Mohri09}. 

Technically, our results rely on a simple construction of the ghost sample that enjoys the necessary properties, and, for Rademacher complexity-based bounds, a concentration inequality adapted to dependent variables by following \cite{vanDeGeer07}. Specifically, these allow us to prove the following claims along the different sections of the paper.
\begin{itemize}
	\item The standard classification risk bounds  of \cite{Vapnik98} that are based on the VC-dimension hold in the exact same form for non-i.i.d. training sequences. Interestingly, no new concentration result is required for the proof (Sect.~\ref{sec:VCbound}). 
	\item With an additional stationarity assumption, similar conclusions hold for VC relative deviation bounds and classification risk bounds with fast rates, thus generalizing the results of \cite{Vapnik98,Cortes19} (Sect.~\ref{sec:relativedeviation}).
	\item Extending the results above to deal with regression problems poses no difficulty in the non-i.i.d. context (Sect.~\ref{sec:VCreg}).	
	\item When the Rademacher complexity can be bounded in terms of the marginal distributions of the data, Rademacher complexity-based bounds identical to the ones for the standard i.i.d. case hold for dependent data. This is true in particular for linear and kernel machines, thus generalizing the results of \cite{Bartlett02} (Sect.~\ref{sec:rad}).
	\item Using the chaining method, standard risk bounds based on uniform covering numbers or the fat-shattering dimension are shown to hold also for non-i.i.d. data (Sect.~\ref{sec:chaining}).
	\item In the framework of robust optimization, the sample complexities of random programs derived in \cite{Alamo09,Lauer23} for independent scenarios also hold with dependent scenarios (Sect.~\ref{sec:scenarioopt}).  
\end{itemize}

Finally, the main drawback of our approach is that it does not allow the derivation of data-dependent bounds, i.e., risk bounds in which the complexity term is evaluated with respect to the available training sample. Data-dependent bounds may be tighter than worst-case or average estimates. However, this improvement remains limited in many standard cases. 
In addition, keeping in mind our motivation stemming from the field of scenario-based optimization, we note that data-dependent bounds are irrelevant for computing sample complexity estimates. 

\section{Preliminaries}

This section presents the general learning framework with dependent data sequences and the basic tools needed to derive our main results. 

\subsection{Learning framework}
\label{sec:framework}

Let $(\A_i)_{0\leq i\leq n}$ denote a filtration and $\g Z_n=(Z_i)_{1\leq i\leq n}$ a sequence of random variables $Z_i=(X_i,Y_i)\in \X\times \Y=\Z$ adapted to $(\A_i)_{1\leq i\leq n}$,\footnote{A filtration is a sequence of increasing $\sigma$-algebras $\emptyset=\A_0 \subset \A_1\subset \A_2\subset\dots\subset\A_n$ and a sequence of random variables is adapted to it if each $Z_i$ is $\A_i$-measurable.} without any other assumption. Given a model class $\F$ of functions from $\X\to\Y$ and a bounded loss function $\ell:\Y^2 \to [0,B]$, we consider learning a model $f\in\F$ from such a training sequence $\g Z_n$ and aim at the estimation of its risk,
\begin{equation}\label{eq:risk}
	L_n(f) = \frac{1}{n} \sum_{i=1}^n\E \ell(f(X_i), Y_i), 
\end{equation}
where $\E$ denotes the expectation, 
by its empirical risk 
\begin{equation}\label{eq:empiricalrisk}
	\hat{L}_n(f) = \frac{1}{n} \sum_{i=1}^n \ell(f(X_i), Y_i).
\end{equation}

Note that for stationary sequences, the $Z_i$ variables are identically distributed and the risk~\eqref{eq:risk} merely boils down to the standard learning risk $\E \ell(f(X_1), Y_1)$, also considered for sequential data, e.g., in \cite{Mohri09}. 

For non-stationary sequences, the risk in~\eqref{eq:risk} does not really assess the ability of the model to predict future values of $Y_i$, since these need not share the same distribution with the $Y_i$'s in~\eqref{eq:risk}. However, it remains a valuable quantity for other applications, such as vector quantization or clustering problems, where, after the obvious reformulation of the loss as $\ell:\X^2\to[0,B]$, the risk in~\eqref{eq:risk} stands for the distortion or the clustering risk and evaluates how well the model approximates the distribution of the training sample.

Regarding prediction problems, note that the risk~\eqref{eq:risk} differs from the conditional risk
\begin{equation}\label{eq:conditionalrisk}
	\frac{1}{n} \sum_{i=1}^n \E\left[ \ell(f(X_i), Y_i) \mid \A_{i-1} \right]
\end{equation}
studied, e.g., in \cite{Kuznetsov15,Rahklin15}. 
The conditional risk~\eqref{eq:conditionalrisk} only measures the ability of the model to predict the (immediate) future of the training sequence, which is particularly well-suited for time-series forecasting problems. However, it is not suitable for applications such as dynamical system identification where the model is learned from training data offline, and then applied to predict the output of the system that might have been reset in the meantime. In other words, the conditional risk~\eqref{eq:conditionalrisk} refers only to the sample path of the random process used for training, whereas the risk in~\eqref{eq:risk} averages over all sample paths.

\subsection{Basic tools}

Most results in statistical learning theory rely on concentration arguments, which are easily available for samples of independent variables. While we will see in Sect.~\ref{sec:VCbound} that these are sufficient to extend some standard risk bounds to dependent sequences, others will require a concentration inequality for samples of dependent variables.
It is stated here as a specification of the results in \cite{vanDeGeer07}. The detailed proof can be found in Appendix~\ref{sec:proofboundeddiff} for completeness.  

\begin{theorem}[Bounded difference inequality for sequences of dependent variables]\label{thm:boundeddiff}
Let $(Z_i)_{1\leq i\leq n}$ denote a sequence of random variables (not necessarily stationary) taking values in $\Z$ and $g$ be a real-valued function of $Z_1$, ..., $Z_n$ such that it is $\A_n$-measurable and
$$
	\sup_{\substack{(z_j)_{1\leq j\leq n}\in\Z^n\\ z'\in\Z}} \vert g(z_1,\dots, z_i,\dots,z_n) - g(z_1,\dots, z',\dots,z_n)\vert \leq c_i,\quad i=1,\dots, n.
$$
Then, for any $\epsilon>0$,
$$
	P\left\{g(Z_1,\dots,Z_n) - \E g(Z_1,\dots,Z_n) > \epsilon \right\} \leq \exp\left(\frac{-2 \epsilon^2 }{\sum_{i=1}^n c_i^2}\right).
$$
\end{theorem}
Note that Theorem~\ref{thm:boundeddiff} provides the exact same result as McDiarmid's inequality for i.i.d. variables \citep{McDiarmid89}, with the same exponential rate. 

We also recall a famous result from \cite{Hoeffding63}, in its original form for {\em independent} variables, which is all we will require in the sequel. 
\begin{theorem}[Hoeffding's inequality]
\label{thm:hoeffding}
Let $(W_i)_{1\leq i\leq n}$ denote a sequence of independent random variables satisfying $W_i\in [a_i,b_i]$. Then, for any $\epsilon>0$,
$$
	P\left\{ \frac{1}{n}\sum_{i=1}^n W_i -  \frac{1}{n}\sum_{i=1}^n \E W_i > \epsilon \right\} \leq \exp\left(\frac{-2 n^2\epsilon^2 }{ \sum_{i=1}^n (b_i-a_i)^2}\right).
$$
\end{theorem}

\section{Classification risk bounds based on the VC-dimension}
\label{sec:VCbound}
We first discuss how to derive the standard VC classification risk bounds for dependent data.
The first ingredient of such bounds is a symmetrization lemma, which in turn requires two things: a concentration inequality and a ghost sample. 
In the classical scenario where the training sample $\g Z_n$ is made of independent copies of some random variable $Z$, the ghost sample $\g Z_n'$ can merely be taken as an independent copy of $\g Z_n$. Then, standard concentration inequalities apply to this sample of i.i.d. variables $Z_i'$. 

In the case of a sequence of dependent variables, we instead build the ghost sample with variables $Z_i'$ taken as independent copies of the variables $Z_i$. Thus, we obtain a sample $\g Z_n'$ of independent variables that is also independent of $\g Z_n$ and with each $Z_i'\sim Z_i$ (the notation $A\sim B$ indicates that the two random variables $A$ and $B$ share the same distribution). Note however that the resulting ghost sample $\g Z_n'$ need not share the same joint distribution with $\g Z_n$ as in the classical i.i.d. scenario. But, a careful look at the proof of the standard symmetrization lemma of \cite{Vapnik98} shows that this is not needed for symmetrization to hold. In addition, in this proof, the concentration inequality is only applied to the ghost sample, not the training sample. Since our ghost sample is made of independent variables by construction, we have the following (detailed proof in App.~\ref{sec:proofsymmetrization}). 

\begin{lemma}[Symmetrization]\label{lem:VCsymmetrization}
Let $\g Z_n$ denote a sequence of possibly dependent variables and consider the ghost sample $\g Z_n'=(Z'_i)_{1\leq i\leq n}$ with each $Z'_i=(X'_i,Y'_i)$ built as an independent copy of $Z_i=(X_i,Y_i)$. Then, for any $\epsilon$ such that $n\epsilon^2\geq 2B^2$,
$$
P\left\{ \sup_{f\in\F} L_n(f) - \hat{L}_n(f) \geq \epsilon\right\} 
	\leq 2 P\left\{ \sup_{f\in\F} \hat{L}'_n(f) - \hat{L}_n(f) \geq \frac{\epsilon}{2}\right\} ,
$$
where $\hat{L}'_n(f) = \frac{1}{n}\sum_{i=1}^n \ell(f(X_i'), Y_i')$.
\end{lemma}

Equipped with this tool, 
classification risk bounds based on the zero-one loss\footnote{With the loss~\eqref{eq:zerooneloss}, the risk~\eqref{eq:risk} is just the probability of misclassification $P(f(X)\neq Y)$ for stationary sequences.}
\begin{equation}\label{eq:zerooneloss}
	\ell(\hat{y},y) = \I{\hat{y}\neq y}
\end{equation}
can be obtained in terms of the growth function $\Pi_{\F}(n)$ of the class $\F$ that can be bounded by the VC-dimension $d_{VC}$. 

\begin{definition}[Growth function]\label{def:growthfunction}
For a set $\F$ of classifiers $f:\X\to\Y$ with a discrete set $\Y$ of finite cardinality, its {\em growth function} is the largest cardinality of the set of classifications produced by its classifiers over all sets $\g x_n=(x_i)_{1\leq i\leq n}$ of $n$ points:
$$
	\Pi_{\F}(n) = \sup_{\g x_n\in\X^n} \lvert \{(f(x_1),\dots,f(x_n)) : f\in\F\}\rvert.
$$
\end{definition}

\begin{definition}[VC-dimension]\label{def:VCdim}
The {\em Vapnik-Chervonenkis (VC) dimension} $d_{VC}$ of a set $\F$ of binary classifiers $f:\X\to\Y$, $\vert\Y\vert =2$, is the largest number of points $n$ such that $\Pi_{\F}(n) =2^n$.
\end{definition}

More precisely, it is a remarkable fact that no new concentration result is needed to derive a risk bound for dependent data from Lemma~\ref{lem:VCsymmetrization}. The trick is that, by introducing independent Rademacher variables and conditioning, we only require concentration with respect to these variables rather than for the samples themselves. Therefore, the classical Hoeffding inequality of Theorem~\ref{thm:hoeffding} applies and the proof of the next result (given in Appendix~\ref{sec:proofVCbound}) only requires that the quantity $\I{f(X_i)\neq Y_i} - \I{f(X'_i)\neq Y'_i}$ is a symmetric random variable, which is ensured by the fact that $(X'_i,Y'_i)$ is an independent copy of $(X_i,Y_i)$.

\begin{theorem}[Basic VC risk bound for dependent data]
\label{thm:VCbound}
Let $\F$ be a set of classifiers $f:\X\to\{-1,+1\}$ with VC-dimension $d_{VC}$ and $\ell$ denote the classification loss~\eqref{eq:zerooneloss}. Then, for any $\delta\in(0,1)$, with probability at least $1-\delta$, 
\begin{align*}
	\forall f\in\F,\quad L_n(f) 
			&\leq \hat{L}_n(f) + 2\sqrt{2\frac{ \log \Pi_{\F}(2n) + \log \frac{2}{\delta}}{n}}\\
			&\leq \hat{L}_n(f) + 2\sqrt{2\frac{d_{VC}\log \frac{2 e n}{d_{VC}} + \log \frac{2}{\delta}}{n}}.
\end{align*}
\end{theorem}

Note that Theorem~\ref{thm:VCbound} applies to dependent training sequences $\g Z_n$ and provides the exact same result as \cite{Vapnik98} did for the i.i.d. case.

\begin{example}[Linear classification]
Consider the set of linear classifiers, $\F=\{ f : f(x) = \sign(\inner{w}{x} + b),\ w\in\R^d,\ b\in\R\}$, of VC-dimension $d_{VC}=d+1$ \citep{Vapnik98}. For all $d\geq 3$, Theorem~\ref{thm:VCbound} then yields for $n=100\,000$ that, with $95\%$ probability, the risk of any classifier $f\in\F$ can be estimated from its empirical risk on the training sample of dependent instances with accuracy no less than  $ 0.031\sqrt{d+1}+0.015$.
\end{example}

\subsection{Relative deviation bounds}
\label{sec:relativedeviation}

Risk bounds with a faster convergence rate close to $O(1/n)$ instead of $O(1/\sqrt{n})$ can be derived from relative deviation bounds of the form
$$
	P\left\{ \sup_{f\in\F} \frac{L_n(f) - \hat{L}_n(f)}{\sqrt{L_n(f)}} \geq \epsilon\right\} \leq A \exp(-nC),
$$
for some constants $A$ and $C$. In turn, obtaining such results requires a different form of symmetrization, associated to a concentration argument. 

The detailed proof of the corresponding symmetrization given in \cite{Cortes19} shows that independence is used at only two different places. It is first required between data points to apply a binomial tail bound, but only on the ghost sample. Then, the fact that the training sample is independent of the ghost sample is used. 

Considering now a construction of the ghost sample as in Theorem~\ref{lem:VCsymmetrization}, we directly obtain the required indepedence: $\g Z_n$ is independent of $\g Z_n'$ and all $Z_i'$ in the ghost sample are independent of each other. This leads to the following generalization of the result of \cite{Vapnik98,Cortes19} to stationary sequences of training data (the detailed proof can be found in Appendix~\ref{sec:proofrelativesymmetrization} for completeness). Note that stationarity, which is required for the binomial tail bound, is a reasonable assumption for a prediction task as considered here.
\begin{lemma}[Symmetrization for relative deviations]\label{lem:relativesymmetrization}
Let $\g Z_n=(Z_i)_{1\leq i\leq n}$ denote a stationary sequence of possibly dependent variables and consider the ghost sample $\g Z_n'=(Z'_i)_{1\leq i\leq n}$ with each $Z'_i=(X'_i,Y'_i)$ built as an independent copy of $Z_i=(X_i,Y_i)$. Let $\ell$ denote the classification loss~\eqref{eq:zerooneloss}. Then, for any $\epsilon$ such that $n\epsilon^2>1$,
$$
P\left\{ \sup_{f\in\F} \frac{L_n(f) - \hat{L}_n(f)}{\sqrt{L_n(f)}} \geq \epsilon\right\} 
	\leq 4 P\left\{ \sup_{f\in\F} \frac{ \hat{L}'_n(f) - \hat{L}_n(f) }{\sqrt{(\hat{L}_n(f)+\hat{L}'_n(f)+\frac{1}{n})/2}}\geq \frac{\epsilon}{2}\right\} ,
$$
where $\hat{L}'_n(f) = \frac{1}{n}\sum_{i=1}^n \I{f(X_i')\neq Y_i'}$.
\end{lemma}

As for Theorem~\ref{thm:VCbound}, we can derive a fast-rate risk bound from Lemma~\ref{lem:relativesymmetrization} using only standard concentration arguments, i.e., by following the proof of \cite{Cortes19} and introducing independent Rademacher variables, and the symmetry of $\I{f(X_i)\neq Y_i} - \I{f(X'_i)\neq Y'_i}$ ensured by our ghost sample construction (detailed proof given in App.~\ref{sec:proofgeneralVCbound}).

\begin{theorem}[General VC risk bound for dependent data]
\label{thm:generalVCbound}
Let $\F$ be a class of classifiers $f:\X\to\{-1,+1\}$ with VC-dimension $d_{VC}$ and  $\ell$ denote the classification loss~\eqref{eq:zerooneloss}. Let $\g Z_n=(Z_i)_{1\leq i\leq n}$ denote a stationary sequence of possibly dependent variables. Then, for any $\delta\in(0,1)$, with probability at least $1-\delta$, 
\begin{align*}
	\forall f\in\F,\quad L_n(f) &\leq \hat{L}_n(f) + 2\sqrt{\hat{L}_n(f)\frac{ \log \Pi_{\F}(2n) + \log \frac{4}{\delta}}{n}} + 4\frac{ \log \Pi_{\F}(2n) + \log \frac{4}{\delta}}{n} \\
			&\leq \hat{L}_n(f) + 2\sqrt{\hat{L}_n(f)\frac{ d_{VC}\log \frac{2 e n}{d_{VC}} + \log \frac{4}{\delta}}{n}} + 4\frac{d_{VC}\log \frac{2 e n}{d_{VC}}+ \log \frac{4}{\delta}}{n}.
\end{align*}
\end{theorem}
Here again, the bounds of Theorem~\ref{thm:generalVCbound} provide the same guarantees as the classical results of \cite{Vapnik98,Cortes19}, while holding more generally for dependent data. In particular, for classifiers achieving a small empirical error $\hat{L}_n(f)$, these are tighter than the ones of Theorem~\ref{thm:VCbound}, with a convergence rate in $O(1/n)$ in the optimistic scenario where a perfect fit of the data can be ensured.

Finally, we note that the fast-rate and margin-based risk bounds provided in~\cite{Cortes21} and that involve covering numbers instead of the growth function can also be proved to hold for dependent data using similar arguments as above. However, we refrain from giving the details and will instead consider margin classifiers when discussing Rademacher complexity-based bounds in Sect.~\ref{sec:rad}.

\subsection{Regression bounds}
\label{sec:VCreg}

When $\ell:\Y^2\to [0,B]$ is a regression loss bounded by $B$, the VC theory for classification applies thanks to Inequality (5.11) in~\cite{Vapnik98}:
\begin{equation}\label{eq:VCreg}
	P\left\{ \sup_{f\in\F} L_n(f) - \hat{L}_n(f) > \epsilon\right\}
	\leq P\left\{ \sup_{f\in\F,\beta\in(0,B)} \tilde{L}(f) - \hat{\tilde L}_n(f) > \frac{\epsilon}{B}\right\},
\end{equation}
where $\tilde{L}(f)$ and $\hat{\tilde{L}}_n(f)$ are defined as in~\eqref{eq:risk}--\eqref{eq:empiricalrisk} with the classification loss $\tilde{\ell}(f(X_i),Y_i) = \I{\ell(f(X_i),Y_i)-\beta\geq 0}$. Interestingly, it can be checked that Inequality~\eqref{eq:VCreg} holds irrespective of the independence of the $Z_i$. Thus, the classification bounds above can be applied to its right-hand side with the VC-dimension of a function class that has one additional parameter $\beta\in(0,B)$, and this yields bounds on the regression risk of models learned from dependent data sequences. 

\begin{example}[VC bound for linear system identification]
Consider a regression problem with data generated by a linear dynamical system as $y_{i} = \inner{\theta}{x_i} + \nu_i$ with a random noise term $\nu_i$ and $x_i=(y_{i-1},\dots,y_{i-d})$. In this case, the data collected from a single trajectory of the system clearly cannot be assumed to be i.i.d., but the results above still yield performance guarantees. Specifically, for the squared loss $\ell(\hat{y},y)=(\hat{y}-y)^2$ and a class of linear models $\F=\{f : f(x) = \inner{w}{x}  ,\ w\in\R^d\}$ over $\X\subset\R^d$, the VC-dimension of the set of functions $\I{\ell(f(x),y)-\beta\geq 0}$ induced by $\F\times (0,B)$ can be computed as follows. Let $\phi$ be the function that maps $x\in\X$ into a vector of all the $d^2$ monomials of degree $2$ over the components of $x$: $\phi_{i+(j-1)d}(x) = x_i x_j$, $i=1,\dots,d$, $j=1,\dots,d$. Let $z=(x,y)$ and $\psi(z) = (\phi(x), yx, y^2)$. Then, for any $f\in\F$ and $\beta\in(0,B)$, 
\begin{align*}
	\ell(f(x),y)-\beta &= (\inner{w}{x} - y)^2-\beta
	= (\inner{w}{x})^2 - 2\inner{w}{yx} + y^2 - \beta,
\end{align*} 
which is a quadratic function of $z$ but a linear function of $\psi(z)$. Therefore, the VC-dimension of the set of functions $\I{\ell(f(x),y)-\beta\geq 0}$ cannot be larger than the VC-dimension of the set of linear classifiers of $\R^{d^2+d+1}$ to which all the $\psi(z)$ belong, i.e., $d_{VC}\leq d^2+d+2$. Thus, Inequality~\eqref{eq:VCreg} ensures with Theorem~\ref{thm:VCbound} that, with probability at least $1-\delta$,
$$
	\forall f\in\F,\quad L_n(f) \leq \hat{L}_n(f) + 2B\sqrt{2\frac{(d^2+d+2)\log \frac{2 e n}{d^2+d+2} + \log \frac{2}{\delta}}{n}} .
$$
\end{example}

\section{Rademacher complexity-based bounds}
\label{sec:rad}

We now consider the derivation of risk bounds based on Rademacher complexities instead of VC-dimensions. This also relies on a symmetrization step, which we conduct using a ghost sample defined as in Section~\ref{sec:VCbound}. In addition, we will also require a bounded difference inequality for dependent variables, which we established in Theorem~\ref{thm:boundeddiff}. 

Let us first define the relevant capacity measures. 
\begin{definition}[Rademacher complexities]
Let $\g T_n=(T_i)_{1,\leq i\leq n}$ be a sequence of (not necessarily independent nor identically distributed) random variables $T_i\in\mathcal{T}$ and $\g\sigma_n=(\sigma_i)_{1\leq i\leq n}$ an i.i.d. sequence of uniformly distributed $\sigma_i\in\{-1,+1\}$. Let $\F$ be a class of real-valued functions over $\mathcal{T}$. The {\em empirical Rademacher complexity} of $\F$ given $\g T_n$ is
$$
	\hat{\mathcal{R}}_{\g T_n}(\F) = \E\left[ \left. \sup_{f\in\F}\frac{1}{n}\sum_{i=1}^n \sigma_i f(T_i) \ \right\vert \g T_n \right]
$$
and the {\em Rademacher complexity} of $\F$ is 
$$
	\mathcal{R}_{\g T_n}(\F) = \E\hat{\mathcal{R}}_{\g T_n}(\F) = \E \sup_{f\in\F}\frac{1}{n}\sum_{i=1}^n \sigma_i f(T_i).
$$
\end{definition}
Note that we keep the subscript $\g T_n$ in the notation of the Rademacher complexity to emphasize that its definition depends on the distribution of $\g T_n$. For a sequence of values $\g t_n=(t_i)_{1\leq i\leq n}\in\mathcal{T}^n$, we will also occasionally write $\hat{\mathcal{R}}_{\g t_n}(\F)$ with a lowercase $\g t_n$ in the subscript to refer to the value of the empirical Rademacher complexity computed given $\g T_n=\g t_n$.

We can now state the first result of this section (proved in App.~\ref{sec:proofradsymmetrization}).
\begin{theorem}[Rademacher complexity-based bound]\label{thm:radbound}
Let $\g Z_n$ denote a sequence of possibly dependent variables and consider the ghost sample $\g Z_n'=(Z'_i)_{1\leq i\leq n}$ with each $Z'_i=(X'_i,Y'_i)$ built as an independent copy of $Z_i=(X_i,Y_i)$. Let $\ell : \Y\times \Y\to [0,B]$ be a loss function. Then, for any $\delta\in(0,1)$, 
with probability at least $1-\delta$,
$$
	\forall f\in\F,\quad   L_n(f) \leq  \hat{L}_n(f) 
	+ \mathcal{R}_{\g Z_n}(\L) +\mathcal{R}_{\g Z'_n}(\L)+ B\sqrt{\frac{ \log \frac{1}{\delta}}{2n}},
$$
where $\L=\{\ell_f : \ell_f(z) = \ell(f(x),y),\ f\in\F\}$, and $\mathcal{R}_{\g Z'_n}(\L)$ is the Rademacher complexity  of $\L$ defined over the ghost sample $\g Z_n'$ instead of $\g Z_n$.
\end{theorem}

The main difference with classical Rademacher complexity-based bounds of, e.g., \cite{Bartlett02}, is due to the possibility of having  different joint distributions for $\g Z_n$ and $\g Z_n'$, which prevents us from collapsing $\mathcal{R}_{\g Z_n}(\L) +\mathcal{R}_{\g Z'_n}(\L)$ into $2\mathcal{R}_{\g Z_n}(\L)$. 
In particular, this compromises the possibility to obtain data-dependent bounds based on the empirical Rademacher complexities, since the ghost sample is not really available to compute $\hat{\mathcal{R}}_{\g Z'_n}(\L)$.

However, even in the standard i.i.d. case where we can perform such a simplification, it is common practice to upper bound the Rademacher complexity by its empirical version in a worst-case manner, which we reproduce here:
$$
\mathcal{R}_{\g Z_n}(\L) \leq \sup_{\g z_n\in\Z^n} \hat{\mathcal{R}}_{\g z_n}(\L)\qquad and\quad 
\mathcal{R}_{\g Z'_n}(\L) \leq \sup_{\g z_n\in\X^n} \hat{\mathcal{R}}_{\g z_n}(\L).
$$
Therefore, we can obtain a risk bound that does not depend on the ghost sample anymore, and which coincides with common bounds for the i.i.d. case: 
\begin{corollary}
For any $\delta\in(0,1)$, with probability at least $1-\delta$,
$$
	\forall f\in\F,\quad L_n(f) \leq \hat{L}_n(f)
	+ 2\sup_{\g z_n\in\Z^n} \hat{\mathcal{R}}_{\g z_n}(\L) + B\sqrt{\frac{\log \frac{1}{\delta}}{2n}}.
$$
\end{corollary}

\begin{example}[Linear and kernel regression]
Consider the squared loss, $\ell(\hat{y},y)=(y-\overline{\hat{y}})^2$, computed from a bounded output $Y\in[-M,M]$ and a clipped prediction $\overline{\hat{y}} = \min(M,\max(-M, \hat{y}))$ given by a linear model $\hat{y}=f(x)$ from $\F_{\text{lin}}=\{f : f(x) = \inner{w}{x},\ \|w\|\leq \Lambda\}$. Then, standard computations (detailed in App.~\ref{sec:radreglinyield}) yield
\begin{equation}\label{eq:radreglin}
	\sup_{\g z_n\in\Z^n} \hat{\mathcal{R}}_{\g z_n}(\L)  \leq \frac{4M\Lambda\sup_{x\in\X}\|x\|}{\sqrt{n}},
\end{equation}
which can only be improved in data-dependent bounds through the substitution of $\sqrt{\sum_{i=1}^n \|X_i\|^2 / n}$ for $\sup_{x\in\X}\|x\|$. 
If we now consider $\F$ as the ball of radius $\Lambda$ in the Reproducing Kernel Hilbert Space (RKHS) $\H$ induced by the Gaussian kernel, $K(x,x') = \exp(-\|x-x'\|/2\sigma^2)$, we have instead 
$$
	\sup_{\g z_n\in\Z^n} \hat{\mathcal{R}}_{\g z_n}(\L)  \leq \frac{4M\Lambda}{\sqrt{n}},
$$
which is exactly what the corresponding data-dependent bound would give in an i.i.d. setting. 
\end{example}

Instead of relying on worst-case estimates, another possibility is to use an upper bound on the Rademacher complexity expressed in terms of the marginal distributions of the variables $Z_i$, which coincide with those of the $Z_i'$. 

\begin{theorem}
\label{thm:radboundmarginals}
Let $\g Z_n$ denote a sequence of possibly dependent variables $Z_i=(X_i,Y_i)\in\X\times \Y$. For a bounded loss function $\ell : \Y\times\Y\to[0,B]$ and a class $\F$ of functions from $\X$ to $\Y$, let $\L=\{\ell_f : \ell_f(z) = \ell(f(x),y),\ f\in\F\}$. If 
\begin{equation}\label{eq:boundmarginals}
	\mathcal{R}_{\g Z_n}(\L)\leq \overline{\mathcal{R}}_{n}(\L)
\end{equation}
with an upper bound $\overline{\mathcal{R}}_{n}(\L)$ on the Rademacher complexity of $\L$ expressed in terms of the marginal distributions of the $Z_i$'s, then, for any $\delta\in(0,1)$, with probability at least $1-\delta$,
$$
	\forall f\in\F,\quad L_n(f) \leq \hat{L}_n(f)
	+2 \overline{\mathcal{R}}_{n}(\L) + B\sqrt{\frac{\log \frac{1}{\delta}}{2n}}.
$$ 
\end{theorem}
\begin{proof}
Since $\overline{\mathcal{R}}_{n}(\L)$ only involves the marginal distributions of the $Z_i$, which are identical to the ones of the $Z_i'$ by construction of the ghost sample in Theorem~\ref{thm:radbound}, the bound in~\eqref{eq:boundmarginals} also implies $\mathcal{R}_{\g Z'_n}(\L)  \leq \overline{\mathcal{R}}_{n}(\L)$. Then, both Rademacher complexities in the bound of Theorem~\ref{thm:radbound} can be upper bounded by $\overline{\mathcal{R}}_{n}(\L)$, which yields the desired result.
\end{proof}

\begin{example}[Linear and kernel regression---continued]
\label{ex:linearregrad}
Consider again the squared loss with the class $\F_{\text{lin}}$ for linear regression. Then, the classically used upper bound on the Rademacher complexity is precisely of a form suitable for Theorem~\ref{thm:radboundmarginals}:
$$
	\mathcal{R}_{\g Z_n}(\L)\leq \frac{4M \Lambda \sqrt{\sum_{i=1}^n \E \|X_i\|^2}}{n} = \overline{\mathcal{R}}_{n}(\L),
$$
and, since  $X_i'\sim X_i$, 
$$
	\mathcal{R}_{\g Z'_n}(\L)\leq \frac{4M \Lambda \sqrt{\sum_{i=1}^n \E \|X'_i\|^2}}{n} = \frac{4M \Lambda \sqrt{\sum_{i=1}^n \E \|X_i\|^2}}{n} =\overline{\mathcal{R}}_{n}(\L).
$$
Thus, we obtain that, with probability at least $1-\delta$, 
$$
	\forall f\in\F_{\text{lin}},\quad L_n(f) \leq \hat{L}_n(f)
	+\frac{8 M \Lambda \sqrt{\sum_{i=1}^n \E \|X_i\|^2}}{n} + 4M^2\sqrt{\frac{\log \frac{1}{\delta}}{2n}},
$$ 
which coincides with the standard result for linear regression with i.i.d. data. In addition, similar results can also be derived for classes of kernel models with $\E \|X_i\|^2$ replaced by $\E K(X_i,X_i)$, which only depends on the marginal distribution of $X_i$. 
\end{example}

\begin{example}[Margin classifiers]\label{ex:marginclassif}
Binary margin classifiers are classifiers that implement $f(x)=\sign(g(x))$ with a real-valued function $g\in\G$. For those, the piecewise linear margin loss function is usually considered: for a margin $\gamma>0$, 
$$
	\ell_{\gamma}(g(x),y) = \min\{\gamma, \max\{0, \ (1-y g(x))/\gamma\}.
$$
Then, the contraction principle \citep{Ledoux91} yields $\mathcal{R}_{\g Z_n}(\L)\leq \mathcal{R}_{\g Z_n}(\G)/\gamma$ and an upper bound on $\mathcal{R}_{\g Z_n}(\G)$ in terms of the marginal distributions also leads to a bound on $\mathcal{R}_{\g Z_n}(\L)$ suitable for Theorem~\ref{thm:radboundmarginals}. For linear classifiers, $\G=\{g : g(x) = \inner{w}{x},\ \|w\|\leq \Lambda\}$, we can use the bound of Example~\ref{ex:linearregrad} to obtain 
$$
	\overline{\mathcal{R}}_{n}(\L)= \frac{\Lambda \sqrt{\sum_{i=1}^n \E \|X_i\|^2}}{\gamma n} ,
$$
and a bound on the probability of misclassification via $P(f(X)\neq Y) = \E\I{f(X)\neq Y} \leq \E\ell_{\gamma}(g(X),Y) = L_n(f)$, where $L_n(f)$ is bounded by Theorem~\ref{thm:radboundmarginals}. Of course, similar results hold for kernel machines by exchanging $\E \|X_i\|^2$ with $\E K(X_i,X_i)$.
\end{example}

Theorem~\ref{thm:radboundmarginals} applies also more broadly to other nonconvex loss functions for which the contraction principle does not apply as directly as in Example~\ref{ex:marginclassif}. Note that, for such losses, the empirical risk minimizer remains most often elusive and the obtention of {\em uniform} risk bounds is critical to allow for their application to the model returned by a training algorithm.

\begin{example}[Vector quantization]
Consider the problem of vector quantization in $\X\subset\H$ for some Hilbert space $\H$ as discussed in \cite{Bartlett98b,Biau08}, where one aims at finding a collection $f=(f_k)_{1\leq k\leq C}\in\H^C$ of $C$ codepoints $f_k$ that can well approximate the observations of $X_i\in\X$. The framework of Sect.~\ref{sec:framework} can be modified in the obvious manner to account for this setting by letting $Z_i=X_i$ and computing the loss from $\X^2\to[0,B]$ instead of $\Y^2\to[0,B]$.
For nearest neighbors quantizers, the error of the model $f$ is computed by the loss function
$$
	\ell(f,x) = \min_{k\in\{1,\dots,C\}} \|x - f_k\|^2,
$$
and the risk~\eqref{eq:risk} is the so-called {\em distortion of $f$}. Risk bounds in this setting were derived for the i.i.d. scenario via the Rademacher complexity in \cite{Biau08,Lauer20b}, and their computations led to bounds in terms of the marginals. For instance, for $\F=\F_0^C$, $\F_0=\{f_0 \in\H : \|f_0\|_{\H}\leq \Lambda\}$, we have
$$
	\mathcal{R}_{\g X_n}(\L) \leq \frac{2C \Lambda \sqrt{\sum_{i=1}^n \E \|X_i\|^2}}{n} + \frac{C\Lambda^2}{\sqrt{n}}= \overline{\mathcal{R}}_{n}(\L)
$$ 
and therefore a risk bound that applies in the non-i.i.d. setting via Theorem~\ref{thm:radboundmarginals}.  
\end{example}

Other examples, such as in switching regression or subspace clustering \citep{Lauer20,Lauer20b}, can be found and lead to the same conclusion: in all these settings, common bounds on the Rademacher complexity is expressed in terms of the marginal distributions, and can thus be used verbatim to produce risk bounds with dependent data. For switching regression, this leads to guarantees on the accuracy of models in switched system identification \citep{Lauer19} that are tighter than those derived in \cite{Massucci22} with mixing arguments.

More generally, 
Theorem~\ref{thm:radboundmarginals} can be compared with the results of \cite{Mohri09} that provide generic risk bounds based on Rademacher complexities and mixing arguments, used for instance by \cite{Massucci22}. For a training sequence $\g Z_n$ generated by a $\beta$-mixing process of mixing coefficient $\beta(a)$ that measures the degree of dependence between data points separated by $a$ time steps (see \cite{Bradley05} for proper definitions of these terms), and in a setting otherwise similar to Theorem~\ref{thm:radboundmarginals}, \cite{Mohri09} provides the bound:
\begin{equation}\label{eq:radboundmixing}
	\forall f\in\F,\quad L_n(f) \leq \hat{L}_n(f)
	+2 \overline{\mathcal{R}}_{\mu}(\L) + B\sqrt{\frac{\log \frac{1}{\delta - 4(\mu-1)\beta(a)}}{2\mu}},
\end{equation}
where $a$ and $\mu$ are positive integers such that $n=2a\mu$, and $\delta$ must satisfy $\delta >4(\mu-1)\beta(a)$. Theorem~\ref{thm:radboundmarginals} offers several advantages over the mixing bound~\eqref{eq:radboundmixing}. First, the mixing bound involves the so-called ``effective sample size" $\mu = n/2a < n$ instead of $n$ itself, which makes it less tight, even when focusing only on the complexity term of the bound and $\overline{\mathcal{R}}_{\mu}(\L)$ that is computed on a subset of $\mu$ points. Second, it is only applicable with values of the confidence index $\delta$ larger than a quantity that grows with $\mu$, which prevents us from reaching the ``practical certainty" of $\delta\approx 10^{-9}$ that is often needed in applications such as scenario-based optimization. Finally, the last term of~\eqref{eq:radboundmixing} involves the mixing coefficient $\beta(a)$ that is often unknown and difficult to estimate from the data itself (see \cite{Mcdonald15} for a discussion of this topic).

\subsection{Chaining and covering numbers}
\label{sec:chaining}

Chaining \citep{Talagrand14} is a generic method for estimating the empirical Rademacher complexity of a function class from its covering numbers, defined in terms of an empirical pseudo-metric. 

\begin{definition}[Pseudo-metric]
\label{def:pseudometric}
Given a sequence $\g t_n \in \mathcal{T}^n$, $d_{2,\g t_n}$ is the empirical pseudo-metric over the set of functions from $\mathcal{T}$ to $\R$ defined by
$$
	d_{2,\g t_n}(f,f') = \left( \frac{1}{n}\sum_{i=1}^n \vert f(t_i) - f'(t_i)\vert^2\right)^{\frac{1}{2}} .
$$
\end{definition}

\begin{definition}[Covering numbers] 
Given a class $\F$ of functions of $\mathcal{T}\to\R$ and a pseudo-metric $\rho$, the {\em covering number} $\mathcal{N}(\epsilon, \F, \rho)$ at scale $\epsilon$ of $\F$ for the distance $\rho$ is the smallest cardinality of the proper $\epsilon$-net $\H\subseteq\F$ of $\F$ such that $\forall f\in\F$, $\rho(f,\H) < \epsilon$.
{\em Uniform covering numbers} are defined for the pseudo-metric of Definition~\ref{def:pseudometric} by
$$
	\mathcal{N}_2(\epsilon, \F, n) = \sup_{\g t_n\in \mathcal{T}^n} \mathcal{N}(\epsilon, \F, d_{2,\g t_n} ).
$$
\end{definition}

\begin{theorem}[Chaining]
\label{thm:chaining}
Let $\F$ be a real-valued function class over $\mathcal{T}$ and, for any $\g t_n\in\mathcal{T}^n$, let $D_{\F} = \sup_{(f,f')\in\F^2} d_{2,\g t_n}(f,f')$ denote its diameter. Then, for any integer $N>0$, 
$$
	\hat{\mathcal{R}}_{\g t_n}(\F) \leq \frac{D_{\F}}{2^{N} } + 6 D_{\F} \sum_{j=1}^N 2^{-j} \sqrt{\frac{\log\mathcal{N}(D_{\F}2^{-j} , \F, d_{2,\g t_n})}{n}} .
$$
\end{theorem}

As we have seen in the examples above, the Rademacher complexity of the loss class $\L$ can usually be bounded in terms of the one of the function class $\F$ via contraction arguments that merely introduce a factor 
$L_{\varphi}$.\footnote{The factor $L_{\varphi}$ is the Lipschitz constant of the univariate function $\varphi$ used to compute the loss $\ell(f(x), y)=\varphi(u)$, where $u$ is for instance $yg(x)$ for margin classification or $f(x)-y$ for regression.} Thus, if we let $C(\g X_n,\F)$ denote the bound on the empirical Rademacher complexity of $\F$ given by Theorem~\ref{thm:chaining}, then after taking expectation and using Jensen's inequality, we obtain
\begin{align}\nonumber
	\mathcal{R}_{\g Z_n}(\L) &\leq L_{\varphi} \E C(\g X_n, \F) \\
	&= \frac{L_{\varphi}D_{\F}}{2^{N}} + 6L_{\varphi} D_{\F}\sum_{j=1}^N 2^{-j}\sqrt{\frac{\log \E\mathcal{N}(D_{\F}2^{-j}, \F, d_{2,\g Z_n})}{n}},\label{eq:chainingexp}
\end{align}
and a bound on the Rademacher complexity given in terms of expected covering numbers. This bound thus depends on the distribution of $\g Z_n$ and would differ on the ghost sample $\g Z'_n$. However, in many cases, the covering numbers are merely bounded by their uniform (worst-case) counterpart, i.e., 
$$
\E\mathcal{N}(D_{\F}2^{-j}, \F, d_{2,\g X_n}) \leq \sup_{\g x_n\in\X^n}\mathcal{N}(D_{\F}2^{-j}, \F, d_{2,\g x_n}) = \mathcal{N}_2(D_{\F}2^{-j}, \F, n),
$$
which are themselves bounded in terms of the fat-shattering dimension (see, e.g., \cite{Alon97,Mendelson02,Mendelson03}). 
Then, the resulting bound,
$$
\overline{\mathcal{R}}_{n}(\F) =  \frac{L_{\varphi}D_{\F}}{2^{N}} + 6L_{\varphi}D_{\F}\sum_{j=1}^N 2^{-j}\sqrt{\frac{\log \mathcal{N}_2(D_{\F}2^{-j}, \F, n)}{n}},
$$
on the Rademacher complexity on the training sample applies similarly to the one on the ghost sample, $\mathcal{R}_{\g Z'_n}(\F)$, and this leads, with Theorem~\ref{thm:radboundmarginals}, to the same result one would have obtained under an i.i.d. assumption.

In addition, if one is reluctant to use the uniform (worst-case) covering numbers, many bounds in the literature on the expected covering numbers are given in terms of the marginal distributions. 

For instance, in the multi-category classification setting,~\cite{Bartlett17} derived a covering numbers bound for a class $\L$ of margin loss functions induced by a class of spectrally-regularized deep neural networks, which is of the form
$$
	\log\mathcal{N}(\epsilon, \L, d_{2,\g Z_n}) \leq \frac{A\sum_{i=1}^n\|X_i\|^2}{\epsilon^2}.
$$
This directly yields a bound on the expected log covering numbers that could be used in~\eqref{eq:chainingexp} and that is expressed in terms of the marginal distributions, and thus is valid for both the training and the ghost samples.\footnote{Note that when formulating~\eqref{eq:chainingexp}, we could have applied Jensen's inequality only to the square root and express the bound in terms of expected log coverging numbers.}  

\section{Application to scenario-based optimization}
\label{sec:scenarioopt}

A robust optimization program can typically be written as 
\begin{align}\label{eq:robust}
 & \min_{\theta\in\Theta} J(\theta)\\
	\text{s.t.}\ & f(x,\theta) \leq 0 ,\quad \forall x\in\X,\nonumber
\end{align} 
with an infinite number of constraints induced by the set $\X$ that accounts for the uncertainties in the parameters of the problem. 
Scenario-based optimization computes solutions to~\eqref{eq:robust} with probabilistic guarantees of feasibility by solving random programs of the form
\begin{align}\label{eq:randprog}
	\hat{\theta} &\in\argmin_{\theta\in\Theta} J(\theta)\\
	\text{s.t.}\ & f(X_i,\theta) \leq 0 ,\quad i=1,\dots, n,\nonumber
\end{align}
where the uncertainties are sampled in order to yield a finite number of constraints. 
Note that in order to avoid confusion, we kept the notations of the rest of the paper and denote the random quantity as $X$ and the optimization variable as $\theta$ (instead of $x$ as is usually done in the robust optimization literature). Here, $\g X_n=(X_i)_{1\leq i\leq n}$ is a random sample  of ``scenarios" $X_i\in\X$ and $\Theta$ is a general and deterministic feasible set (think of $\Theta=\R^d$ for instance). 

The two critical issues in scenario-based optimization are to estimate the probability of violation of the computed solution,
$$
	L_n(\hat{\theta}) = P\{ f(X_1, \hat{\theta}) > 0 \},
$$
where we assume stationarity of $\g X_n$, and the number of scenarios one should sample to guarantee a certain reliability, i.e., the sample complexity of~\eqref{eq:randprog}, defined for given $\epsilon$, $\delta\in(0,1)$ as the smallest $n$ such that  
\begin{equation}\label{eq:pacscenario}
	P\{ L_n(\hat{\theta}) > \epsilon \} \leq \delta
\end{equation}
holds.

For independent scenarios, \cite{Alamo09} proposed upper bounds on the probability of violation and sample complexity estimates based on the VC-dimension of the set of indicator functions 
\begin{equation}\label{eq:setindicator}
	\mathcal{I} = \{I_{\theta} \in\{0,1\}^{\X} : I_{\theta}(x) = \I{f(x,\theta) > 0} ,\ \theta\in\Theta  \} .
\end{equation}
Random programs for which the VC-dimension cannot be bounded can be handled with other results from \cite{Lauer23} that are instead based on the Rademacher complexity of the class of margin loss functions, 
\begin{equation}\label{eq:setmarginloss}
	\L_{\gamma,\Theta} = \{\ell_{\theta} \in[0,1]^{\X} : \ell_{\theta}(x) = \min\{1,\max\{0, 1 + f(x,\theta) /\gamma\}\} ,\ \theta\in\Theta  \} ,
\end{equation}
for a margin $\gamma>0$.  

We generalize these results to the case of dependent scenarios below.

\subsection{Sample complexity of random programs with dependent scenarios and finite VC-dimension}

We now consider the case of dependent scenarios, i.e., when $\g X_n$ is a sample of dependent variables. Building on the results of Section~\ref{sec:VCbound}, we obtain the following sample complexity, which we only state for the zero-error case that is rather standard in the framework of scenario optimization. 
\begin{theorem}\label{thm:VCscenario}
Let $\g X_n$ denote a stationary sequence of random variables of values in $\X^n$. Let $d_{VC}$ denote the VC-dimension of the set~\eqref{eq:setindicator}. Assume that, for any $\g x_n\in\X^n$,  the random program~\eqref{eq:randprog} is feasible and that an algorithm computes a feasible point $\hat{\theta}$. Then, for any $\delta\in(0,1)$, with probability at least $1-\delta$,
$$
	P\{ f(X_1, \hat{\theta}) > 0 \} \leq \frac{4 \log \Pi_{\mathcal{I}}(2n) + \log \frac{4}{\delta}}{n}
		\leq \frac{4 d_{VC}\log (\frac{2\e n}{d_{VC}}) + \log \frac{4}{\delta}}{n},
$$
and, for any $\epsilon\in(0,1)$, $\delta\in(0,1)$, the sample complexity of~\eqref{eq:randprog} is no more than 
$$
	n(\epsilon,\delta) = \frac{5}{\epsilon} \left( d_{VC}\log \frac{40}{\epsilon} + \log \frac{4}{\delta} \right).
$$
\end{theorem}  
\begin{proof}
Theorem~\ref{thm:generalVCbound} can be reformulated for the loss $\ell(\theta,x) = \I{f(x,\theta) > 0}$ where $f\in\F$ is identified with $\theta\in\Theta$ and the growth function and VC-dimension of $\F$ are replaced by those of $\mathcal{I}$~\eqref{eq:setindicator}. This yields the bound on the probability of violation above under the assumption that $\hat{\theta}$ is a feasible point for~\eqref{eq:randprog}, i.e., $\hat{L}_n(\hat{\theta}) = 0$. 

Then,~\eqref{eq:pacscenario} holds as soon as 
$$
	 \frac{4 \log \Pi_{\mathcal{I}}(2n) + \log \frac{4}{\delta}}{n} \leq  \epsilon ,
$$
which is implied by
$$
	4 \left(\frac{2\e n}{d_{VC}}\right)^{d_{VC}}\exp(- n\epsilon/4) \leq \delta.
$$
Using Theorem~6 in \cite{Alamo09}, this can be ensured by setting 
$$
	n\geq \inf_{\mu>1}\frac{4}{\epsilon}\frac{\mu}{ \mu - 1} \left( \log \frac{4}{\delta} + d_{VC}\log \frac{8\mu}{\epsilon} \right),
$$
in which the suboptimal choice $\mu=5$ suffices to conclude.
\end{proof}

Here, Theorem~\ref{thm:VCscenario} yields the same guarantees as \cite{Alamo09} for scenario optimization, but applies more broadly to dependent scenarios.  

\subsection{Margin-based sample complexity of pseudo-linear random programs}

For random programs with infinite (or merely too large) VC-dimension, we instead follow the margin-based approach of \cite{Lauer23} which, given a margin parameter $\gamma>0$, implements the following random program instead of~\eqref{eq:randprog}:
\begin{align}\label{eq:randprogmargin}
	\hat{\theta} &\in\argmin_{\theta\in\Theta} J(\theta)\\
	\text{s.t.}\ & f(X_i,\theta) \leq -\gamma ,\quad i=1,\dots, n.\nonumber
\end{align}
More specifically, we focus here on problems in which the constraint function is of the form 
\begin{equation}\label{eq:fmax}
	f(x,\theta) = \max_{k\in\{1,\dots,C\}} f_k(x,\theta)
\end{equation}
for a collection of $C$ functions 
\begin{equation}\label{eq:fk}
	f_k(x,\theta) = \psi_k(x)^\top \phi_k(\theta) + \eta_k(x),\quad k=1,\dots,C,
\end{equation}
based on functions $\psi_k:\X\to\R^{n_k}$, $\phi_k:\Theta\to\R^{n_k}$ and $\eta_k:\X\to\R$.\footnote{The results actually hold for a more general form of composition in~\eqref{eq:fmax}, as detailed in~\cite{Lauer23}.} Then, combining the contraction principle (with Lipschitz constant $1/\gamma$) and Theorem~2 in \cite{Lauer23} yields
\begin{equation}\label{eq:radboundscenario}
	\mathcal{R}_{\g X_n}(\L_{\gamma,\Theta}) \leq \overline{\mathcal{R}}_n(\L_{\gamma,\Theta}) = \frac{1}{\gamma}\sum_{k=1}^C \frac{\tau_k \Lambda_k}{\sqrt{n}},
\end{equation}
where $\L_{\gamma,\Theta}$ is as in~\eqref{eq:setmarginloss}, $\tau_k = \sup_{x\in\X} \|\psi_k(x)\|$, and $\Lambda_k =\sup_{\theta\in\Theta} \|\phi_k(\theta)\|$. Note that this bound holds similarly for any ghost sample $\g X_n' \in\X^n$, which leads to the following.
\begin{theorem}\label{thm:radscenario}
Let $\g X_n$ denote a stationary sequence of random variables of values in $\X^n$. Let $\tau_k = \sup_{x\in\X} \|\psi_k(x)\|$ and $\Lambda_k =\sup_{\theta\in\Theta} \|\phi_k(\theta)\|$. Fix the margin parameter $\gamma>0$ and assume that, for any $\g x_n\in\X^n$,  the random program~\eqref{eq:randprogmargin} is feasible and that an algorithm computes a feasible point $\hat{\theta}$ satisfying the margin condition $f(x_i,\hat{\theta}_i) < -\gamma$, $i=1,\dots,n$. Then,
for any $\delta\in(0,1)$, with probability at least $1-\delta$, 
$$
	P\{ f(X,\hat{\theta})>0 \} \leq \frac{2}{\gamma}\sum_{k=1}^C \frac{\tau_k \Lambda_k}{\sqrt{n}} + \sqrt{\frac{\log \frac{1}{\delta}}{2n}},
$$
and, for any $\epsilon\in(0,1)$, $\delta\in(0,1)$, the sample complexity of~\eqref{eq:randprog} is no more than 
$$
	n(\epsilon,\delta) = \frac{1}{\epsilon^2}\left(\frac{2}{\gamma}\sum_{k=1}^C \tau_k \Lambda_k + \sqrt{\log \frac{1}{\delta}}\right)^2 .
$$
\end{theorem}
\begin{proof}
Apply Theorem~\ref{thm:radbound} to the loss class~\eqref{eq:setmarginloss} and use~\eqref{eq:radboundscenario} to bound the two Rademacher complexities. Then, the bound on the probability of violation follows from the feasibility of $\hat{\theta}$, that ensures that $\hat{L}_n(\hat{\theta})$ computed with the margin loss $\ell_{\gamma}(\theta,x)= \min\{1,\max\{0, 1 + f(x,\theta) /\gamma\}\}$ is zero, and the fact that $P\{ f(X,\hat{\theta})>0 \} \leq P\{ f(X,\hat{\theta})>-\gamma \} = L_n(\hat{\theta})$.

Then,~\eqref{eq:pacscenario} holds as soon as
$$
	\frac{2}{\gamma}\sum_{k=1}^C \frac{\tau_k \Lambda_k}{\sqrt{n}} + \sqrt{\frac{\log \frac{1}{\delta}}{2n}} \leq \epsilon,
$$
which is ensured whenever $n\geq n(\epsilon,\delta)$.
\end{proof}

Again, the proposed method yields with Theorem~\ref{thm:radscenario} the same guarantees as \cite{Lauer23} for scenario optimization, but in a more general setting that allows the sampling of dependent scenarios.

\section{Conclusions}

This paper explored the possible extensions of standard risk bounds to dependent data. This resulted in a number of uniform risk bounds applicable to a wide range of problems. Compared with other approaches from the literature for deriving uniform bounds without assuming independence, the proposed method provides tighter bounds and does not rely on a mixing condition difficult to verify in practice. 

Yet, one drawback of the proposed approach is the lack of data-dependent counterparts of the bounds. However, as seen on a number of examples, data-dependent bounds do not offer improved guarantees in several classical cases. In addition, regarding the motivating application of scenario-based optimization, where the focus is on the sample complexity, data-dependent bounds are irrelevant. In this context, we showed that standard results so far limited to the i.i.d. case also hold with dependent scenarios. This provides the basis for a refined analysis of several problems in control theory such as that of \cite{Wang21}, and the basic tool for investigating new sampling strategies in robust optimization.

Another open issue arises in scenario-based optimization, where the standard bounds for {\em convex} random programs \citep{Campi08,Calafiore10} heavily rely on a binomial tail bound and independence. As these follow a proof scheme rather different than that of learning theory bounds, the current work does not allow their extension to dependent scenarios, which would allow for a straightforward strengthening of the results of \cite{Wang21}.

\appendix

\section{Proof of Theorem~\ref{thm:boundeddiff}}
\label{sec:proofboundeddiff}

Let $(W_i)_{1\leq i\leq n}$ be a random sequence adapted to $(\A_i)$, i.e., a sequence of real-valued $\A_i$-measurable random variables. Recall that an $\A_{i-1}$-measurable variable is said to be {\em predictable}.
Define 
$$
	S_n = \sum_{i=1}^n W_i .
$$

We will first state a few intermediate results on generic variables $W_i$ that are necessary for the proof of Theorem~\ref{thm:boundeddiff}.

\begin{lemma}[Hoeffding's lemma, slightly extended]\label{lem:hoeffding}
For predictable random variables $L_i\leq W_i\leq U_i$ with $\E[W_i \mid \A_{i-1}] = 0$, 
$$
	\E [ e^{\beta W_i} \mid \A_{i-1} ] \leq  e^{\beta^2 (U_i-L_i)^2 /8}.
$$
\end{lemma}
\begin{proof}
Replace the constant bounds by $L_i$, $U_i$ and the expectation by the conditional expectation in the standard proof of \cite{Hoeffding63}: by convexity of the exponential, 
$$
e^{\beta W_i} \leq \frac{U_i - W_i}{U_i-L_i}e^{\beta L_i} + \frac{W_i - L_i}{U_i-L_i}e^{\beta U_i}
$$
and
$$
	\E [ e^{\beta W_i} \mid \A_{i-1} ] \leq \E \left[\left. \frac{U_i - W_i}{U_i-L_i}e^{\beta L_i} \right\vert \A_{i-1} \right] + \E\left[\left.\frac{W_i - L_i}{U_i-L_i}e^{\beta U_i}\right\vert \A_{i-1} \right] .
$$
Then, use the fact that the $L_i$ and $U_i$ are predictable and $\E[W_i \mid \A_{i-1}] = 0$: 
\begin{align*}
	\E [ e^{\beta W_i} \mid \A_{i-1} ] &\leq  \frac{U_i - \E [ W_i \mid \A_{i-1} ] }{U_i-L_i}e^{\beta L_i}  + \frac{\E[W_i \mid \A_{i-1}] - L_i}{U_i-L_i}e^{\beta U_i} \\
	&= \frac{U_i}{U_i-L_i}e^{\beta L_i} + \frac{- L_i}{U_i-L_i}e^{\beta U_i}\\
	&= e^{h(\beta(U_i-L_i))}
\end{align*}
with $h(u) = uL_i/(U_i-L_i) + \log(1 + L_i[(1-e^u ) / (U_i-L_i) ] )$. By computing derivatives and the Taylor expansion, we know that $h(u) \leq u^2/8$, which completes the proof.
\end{proof}

\begin{lemma}[Lemma 2.4 in \cite{vanDeGeer07}]\label{lem:supermart}
Assume $\E[W_i \mid \A_{i-1}] = 0$ and $L_i\leq W_i\leq U_i$ with $L_i$ and $U_i$ that are predictable (i.e., $\A_{i-1}$-measurable). For any $\beta>0$, the sequence of random variables
$$
	\zeta_n(\beta) = \exp\left(\beta S_n - \beta^2 \sum_{i=1}^n (U_i-L_i)^2/8\right),\quad n=1, 2, \dots, 
$$
is a supermartingale, i.e., $\E[\zeta_n(\beta) \mid \A_{n-1}] \leq \zeta_{n-1}(\beta)$.
\end{lemma}
\begin{proof}
\begin{align*}
	\E&[\zeta_n(\beta) \mid \A_{n-1}] \\
	&=  \E\left[\exp\left(\beta \sum_{i=1}^n (W_i - \beta (U_i-L_i)^2/8) \right) \mid \A_{n-1} \right] \\
	& = \E\left[\exp\left(\beta ( W_n - \beta (U_n-L_n)^2/8)\right) \prod_{i=1}^{n-1} \exp\left(\beta ( W_i - \beta (U_i-L_i)^2/8)\right) \mid \A_{n-1} \right] 
\end{align*}
Here, the assumed predictability of $L_i$, $U_i$ implies that, for any function $f$, 
$\forall i\leq n-1$, $\E[f(W_i,L_i,U_i) \mid \A_{n-1}] = f(W_i,L_i,U_i)$.
Thus,
\begin{align*}
	\E[\zeta_n(\beta) \mid \A_{n-1}] &= \left(\prod_{i=1}^{n-1} \exp\left(\beta ( W_i - \beta (U_i-L_i)^2/8)\right)\right)\\
	&\qquad\times \E\left[\exp\left(\beta ( W_n - \beta (U_n-L_n)^2/8)\right)  \mid \A_{n-1} \right] \\
	&\leq \prod_{i=1}^{n-1} \exp\left(\beta ( W_i - \beta (U_i-L_i)^2/8)\right)\\
	&= \exp\left(\beta \sum_{i=1}^{n-1} (W_i - \beta (U_i-L_i)^2/8) \right)\\
	&= \zeta_{n-1}(\beta),
\end{align*}
where we used Lemma~\ref{lem:hoeffding} for  the inequality.
\end{proof}

\begin{theorem}[Hoeffding inequality for sums---simplified version of Theorem 2.5 in \cite{vanDeGeer07} with fixed $n$]
\label{thm:hoeffding1}
Assume that $L_i\leq W_i\leq U_i$ for predictable random variables $L_i$ and $U_i$, i.e., they are $\A_{i-1}$-measurable. Also assume that $\E[W_i\mid \A_{i-1}] = 0$. Then, for any $t>0$ and $c>0$, 
$$
	P\left\{ S_n \geq t \ and\ \sum_{i=1}^n (U_i-L_i)^2 \leq c^2 \right\} \leq \exp(-2t^2 /c^2) .
$$
\end{theorem}
\begin{proof}
By Lemma~\ref{lem:supermart}, $\zeta_n(\beta)$ is a supermartingale. Thus, for any $n$, 
$$
	\E \zeta_{n}(\beta) = \E\E[\zeta_{n}(\beta)\mid\A_{n-1}]\leq \E \zeta_{n-1}(\beta) 
$$
and, by induction, 
\begin{align*}
\E \zeta_{n}(\beta)\leq \E \zeta_1(\beta) &= \E \exp(\beta S_1 - \beta^2 (U_1-L_1)^2/8) = \E \exp(\beta W_1 - \beta^2 (U_1-L_1)^2/8) \\
	&= \E [\exp(\beta W_1 - \beta^2 (U_1-L_1)^2/8) \mid \A_{0} ]\qquad (since\ \A_0=\emptyset)\\
	&  \leq 1,
\end{align*}
where we used Lemma~\ref{lem:hoeffding} for the last inequality. 

Now consider the event 
$$
	A = \left\{ S_n > t \ and\ \sum_{i=1}^n (U_i-L_i)^2\leq c^2 \right\}.
$$
Then, since $\zeta_{n}(\beta) \I{A} \leq \zeta_{n}(\beta)$,
$$
	\E [ \zeta_{n}(\beta) \I{A} ] \leq \E [ \zeta_{n}(\beta) ] \leq 1.
$$
But on $A$, we have
$$
	\zeta_{n}(\beta) = \exp\left(\beta S_{n} - \beta^2 \sum_{i=1}^n (U_i-L_i)^2/8\right) \geq\exp(\beta t - \beta^2 c^2/8) ,
$$
so that
\begin{align*}
	\E [ \zeta_{n}(\beta) \I{A} ] \geq \E [ \exp(\beta t - \beta^2 c^2/8) \I{A} ] &=  \exp(\beta t - \beta^2 c^2/8) \E [  \I{A} ] \\
		&= \exp(\beta t - \beta^2 c^2/8) P(A) .
\end{align*}
Thus, 
$$
	P(A) \leq \frac{1}{\exp(\beta t - \beta^2 c^2/8)} = \exp(-\beta t + \beta^2 c^2/8).
$$
Now set $\beta = 4t/c^2$, this yields
$$
	P(A) \leq \exp(-4 t^2/c^2 + 16 t^2  / 8c^2  ) = \exp(-2t^2/c^2 ).
$$
\end{proof}

Now, we are ready to state the proof of Theorem~\ref{thm:boundeddiff}. 

Let us define the random variables $G = g(Z_1,\dots,Z_n)$ and $W_i = {\E[G\mid\A_{i}]} - \E[G \mid\A_{i-1}]$. Then,
\begin{align*}
	g(Z_1,\dots,Z_n) - \E g(Z_1,\dots,Z_n) &= G - \E G \\
	&= \E [G \mid \A_n] - \E[G \mid \A_0]\\ 
	&= \sum_{i=1}^n \E[G \mid \A_{i}] - \E[G \mid \A_{i-1}]\\
	&= \sum_{i=1}^n W_i.
\end{align*}
Therefore, the theorem is just an application of Theorem~\ref{thm:hoeffding1} to $S_n = \sum_{i=1}^n W_i$, after checking that 
\begin{align*}
	\E[W_i \mid\A_{i-1}] &= \E\left[\left.\E[G \mid\A_{i}]\ \right\vert\A_{i-1}\right] - \E\left[\left.\E[G \mid\A_{i-1}]\ \right\vert \A_{i-1}\right] \\
	&= \E[G \mid\A_{i-1}] - \E[G \mid\A_{i-1}] = 0
\end{align*}
(due to the tower property, since $\A_{i-1}\subset\A_i$). Actually, to apply Theorem~\ref{thm:hoeffding1}, we also need predictable ($\A_{i-1}$-measurable) bounds on $W_i$. First note that 
$$
	L_i = \inf_{z} \E\left[ g(Z_1,\dots,Z_{i-1},z,Z_{i+1},\dots,Z_n) \mid \A_{i} \right] 
$$ 
and 
$$
	U_i = \sup_{z} \E\left[ g(Z_1,\dots,Z_{i-1},z,Z_{i+1},\dots,Z_n) \mid \A_{i} \right] 
$$ 
are $\A_{i-1}$-measurable (since $Z_i$ is replaced by an arbitrary $z$ in their definition) and satisfy
$$
	L_i\leq \E[G \mid\A_i]\leq U_i,
$$
while the bounded difference assumption ensures that $U_i-L_i \leq c_i$. So, since $W_i = \E[G \mid\A_{i}] - \E[G \mid\A_{i-1}]$, we have the following predictable bounds:  
$$
	\tilde{L}_i = L_i -\E[G \mid \A_{i-1}] \leq W_i \leq U_i - \E[G \mid\A_i] = \tilde{U}_i
$$
with $\tilde{U}_i-\tilde{L}_i \leq c_i$. Thus, with $c^2 = \sum_{i=1}^n c_i^2$, $\sum_{i=1}^n ( \tilde{U}_i-\tilde{L}_i)^2\leq c^2$ always holds and, by Theorem~\ref{thm:hoeffding1}: for any $t>0$,
$$
	P(S_n>t) = P\left\{S_n > t \ and \ \sum_{i=1}^n (\tilde{U}_i-\tilde{L}_i)^2\leq c^2\right\} \leq \exp(-2t^2 / c^2).
$$
Recalling that $S_n = g(Z_1,\dots,Z_n) - \E g(Z_1,\dots,Z_n)$ and choosing $t=\epsilon$ completes the proof.

\section{Proof of Lemma~\ref{lem:VCsymmetrization}}
\label{sec:proofsymmetrization}

Define the events
$$
	A(f) = \left\{ L_n(f) - \hat{L}_n(f) \geq \epsilon \right\},\qquad 
	A = \left\{\sup_{f\in\F} L_n(f) - \hat{L}_n(f) \geq \epsilon\right\}
$$
$$
	B(f) = \left\{\hat{L}'_n(f) - \hat{L}_n(f)\geq \frac{\epsilon}{2}\right\},\qquad
	 B = \left\{\sup_{f\in\F}  \hat{L}'_n(f) - \hat{L}_n(f) \geq \frac{\epsilon}{2}\right\}.
$$
Let $f^*\in\F$ denote the function that depends solely on $\g Z_n$ and such that\footnote{In case this is not possible, choose $f^*$ as the function for which the difference is $\eta$-close to the supremum and take the limit $\eta\to 0$ to conclude as in \cite{Cortes19}.} 
$$
	L_n(f^*) - \hat{L}_n(f^*) = \sup_{f\in\F} L_n(f) - \hat{L}_n(f) .
$$
Then, since $f^*\in\F$, we have 
$$
	P(B) \geq P\{B(f^*)\}. 
$$
For any real numbers $a$, $b$, $c$, 
$$
(c-a)\geq\epsilon \ \wedge\ (b-c)\geq \frac{-\epsilon}{2}\quad \Rightarrow\quad b-a \geq \frac{\epsilon}{2},
$$
and thus
$$
P(b-a \geq \epsilon/2) \geq P(c-a\geq \epsilon,\ b-c\geq -\epsilon/2) .
$$
Applied with $a=\hat{L}_n(f^*)$, $b=\hat{L}'_n(f^*)$ and $c=L_n(f^*)$, this gives:
\begin{align*}
	P\{B(f^*)\} &\geq P\left\{L_n(f^*) - \hat{L}_n(f^*) \geq \epsilon,\ \hat{L}'_n(f^*) - L_n(f^*)\geq \frac{-\epsilon}{2} \right\}\\
		&= P\left\{A(f^*) ,\ \hat{L}'_n(f^*) - L_n(f^*) \geq \frac{-\epsilon}{2} \right\}\\
		&=\E  \I{A(f^*)} \I{\hat{L}'_n(f^*) - L_n(f^*)  \geq \frac{-\epsilon}{2}} \\
		&= \E \E \left[ \left. \I{A(f^*)} \I{\hat{L}'_n(f^*) - L_n(f^*)  \geq \frac{-\epsilon}{2}} \right\vert \g Z_n \right]\\
		&= \E  \I{A(f^*)}  \E \left[ \left.\I{\hat{L}'_n(f^*) - L_n(f^*) \geq \frac{-\epsilon}{2}} \right\vert \g Z_n \right]\\
		&= \E \I{A(f^*)}  P \left\{ \left.\hat{L}'_n(f^*) - L_n(f^*)  \geq \frac{-\epsilon}{2} \right\vert \g Z_n \right\}.
\end{align*}
The probability in the latter can be bounded as
\begin{align*}
	P \left\{ \left.\hat{L}'_n(f^*) - L_n(f^*) \geq \frac{-\epsilon}{2} \right\vert \g Z_n \right\}
		&= 	P \left\{ \left.\left\vert\hat{L}'_n(f^*) - L_n(f^*) \right\vert \leq \frac{\epsilon}{2} \right\vert \g Z_n \right\} \\
			&\quad + P \left\{ \left.\hat{L}'_n(f^*) - L_n(f^*) > \frac{\epsilon}{2} \right\vert \g Z_n \right\}\\
		&\geq P \left\{ \left.\left\vert\hat{L}'_n(f^*) - L_n(f^*)\right\vert \leq \frac{\epsilon}{2} \right\vert \g Z_n \right\} \\
		&\geq 1 - P \left\{ \left.\left\vert\hat{L}'_n(f^*) - L_n(f^*)\right\vert > \frac{\epsilon}{2} \right\vert \g Z_n \right\} \\
		&\geq \frac{1}{2},
\end{align*}
where the last inequality is obtained from (a conditional form of) Bienaymé--Chebyshev's inequality applied to the random variable $\hat{L}'_n(f^*)$ whose expectation is $L_n(f^*)$ by construction of the ghost sample ($Z_i'\sim Z_i$):
$$
	\E \hat{L}'_n(f^*) = \frac{1}{n}\sum_{i=1}^n \E \ell(f^*(X_i'),Y_i') = \frac{1}{n}\sum_{i=1}^n \E \ell(f^*(X_i),Y_i) = L_n(f^*).
$$
Note that we can apply this inequality here since given $\g Z_n$, $f^*$ is fixed and independent of $\g Z_n'$. Moreover, by the independence of $\g Z_n$ and $\g Z_n'$:
\begin{align*}
	\E\left[ (\hat{L}'_n(f^*) - L_n(f^*))^2 \mid \g Z_n \right] &\leq \sup_{f\in\F} \E\left[ (\hat{L}'_n(f) - L_n(f))^2 \right] \\
	&= \sup_{f\in\F} Var\left[ \hat{L}'_n(f) \right] \\
	&=\frac{1}{n^2}\sup_{f\in\F} \sum_{i=1}^n Var\left[ \ell(f(X'_i), Y'_i) \right] \\
	&\leq \frac{ B^2}{4n},
\end{align*}
where the two last lines are due to the independence of the $Z_i'$ and the fact that, for any $f$, $\ell(f(X'_i), Y'_i)\in[0,B]$.
 Thus,
\begin{align*}
	P \left\{ \left.\left\vert\hat{L}'_n(f^*) - L_n(f^*) \right\vert \geq \frac{\epsilon}{2} \right\vert \g Z_n \right\} 
	&\leq \frac{4 \E\left[ (\hat{L}'_n(f^*) - L_n(f^*))^2 \mid \g Z_n \right]}{\epsilon^2}\\
	&\leq \frac{ B^2}{n\epsilon^2} \leq \frac{1}{2},
\end{align*}
since, by assumption, $n\epsilon^2 > 2B^2$. 

Therefore, we have shown that 
$$
	P(B) \geq P\{B(f^*)\} \geq \frac{1}{2} \E \I{A(f^*)} = \frac{1}{2} P\{A(f^*)\} = \frac{1}{2}P(A)
$$
and this concludes the proof.

\section{Proof of Theorem~\ref{thm:VCbound}}
\label{sec:proofVCbound}

We will prove that 
\begin{equation}\label{eq:vcboundproof1}
P\left\{ \sup_{f\in\F} L_n(f) - \hat{L}_n(f) \geq \epsilon \right\} \leq 2\Pi_{\F}(2n) \exp(-n\epsilon^2 / 8),
\end{equation}
which implies the first inequality of Theorem~\ref{thm:VCbound} if we let $\delta= 2\Pi_{\F}(2n) \exp(-n\epsilon^2 / 8)$ and solve this equation for $\epsilon$. The second inequality is then obtained by bounding the growth function by the VC-dimension with Sauer's lemma \citep{Sauer72,Vapnik98}, as usual.

By Lemma~\ref{lem:VCsymmetrization}, we only have to bound the deviation between the empirical risk on the training sample and the one on the ghost sample. 
Let $\g\sigma_n=(\sigma_i)_{1\leq i\leq n}$ denote a sequence of independent Rademacher variables uniformly distributed in $\{+1,-1\}$. Then, this random quantity, 
$$
	D = \hat{L}_n'(f) - \hat{L}_n(f) = \frac{1}{n} \sum_{i=1}^n \I{f(X_i')\neq Y_i'} - \I{f(X_i)\neq Y_i} ,
$$
is identically distributed to
$$
	D = \frac{1}{n} \sum_{i=1}^n \sigma_i(\I{f(X_i')\neq Y_i'} - \I{f(X_i)\neq Y_i} ),
$$
since $(X_i,Y_i)$ and $(X_i',Y_i')$ are independent and identically distributed, which makes $(\I{f(X_i')\neq Y_i'}- \I{f(X_i)\neq Y_i} )$ a symmetric random variable whose distribution remains unchanged after multiplication by a random sign $\sigma_i$. Thus,
\begin{align*}
	P&\left\{ \sup_{f\in\F} \hat{L}_n'(f) - \hat{L}_n(f) > \frac{\epsilon}{2}\right\} \\
	&= \E \I{}\left\{ \sup_{f\in\F} \frac{1}{n} \sum_{i=1}^n \sigma_i(\I{f(X_i')\neq Y_i'} - \I{f(X_i)\neq Y_i} ) > \frac{\epsilon}{2} \right\} \\
	&= \E \ \E\left[\left.  \I{}\left\{ \sup_{f\in\F} \frac{1}{n} \sum_{i=1}^n \sigma_i(\I{f(X_i')\neq Y_i'} - \I{f(X_i)\neq Y_i} ) > \frac{\epsilon}{2} \right\} \right\vert \g Z_n,\g Z_n'\right]\\
	&= \E\  P \left\{\left.\g\sigma_n :   \sup_{f\in\F} \frac{1}{n} \sum_{i=1}^n \sigma_i(\I{f(X_i')\neq Y_i'} - \I{f(X_i)\neq Y_i} ) > \frac{\epsilon}{2}  \right\vert \g Z_n,\g Z_n'\right\} \\
	&= \E\  P \left\{\left.\g D_n : \sup_{f\in\F} \frac{1}{n} \sum_{i=1}^n D_i > \frac{\epsilon}{2}  \right\vert \g Z_n,\g Z_n'\right\} ,
\end{align*}
where we let $\g D_n=(D_i)_{1\leq i\leq n}$ and $D_i = \sigma_i(\I{f(X_i')\neq Y_i'} - \I{f(X_i)\neq Y_i} )$. For any fixed $f$ and $\g Z_n$, $\g Z_n'$, the conditional expectation of these variables is zero: $\E [ D_i \mid \g Z_n,\g Z_n']=  (\I{f(X_i')\neq Y_i'} - \I{f(X_i)\neq Y_i} )\E \sigma_i = 0$. Thus, by applying Hoeffding inequality (Theorem~\ref{thm:hoeffding}) on the sequence of independent variables $\g D_n$, we obtain
$$
	P \left\{\left.\frac{1}{n} \sum_{i=1}^n D_i > \frac{\epsilon}{2} \right\vert \g Z_n,\g Z_n'\right\}
	 \leq \exp\left( \frac{- n^2 \epsilon^2}{2\sum_{i=1}^n c_i^2} \right),
$$ 
where $c_i= b_i-a_i$ with $a_i\leq D_i \leq b_i$ and $a_i=-1$, $b_i=1$. This gives $c_i^2 = 4$ and 
$$
	\forall f\in\F,\quad P \left\{\left.\frac{1}{n} \sum_{i=1}^n D_i > \frac{\epsilon}{2} \right\vert \g Z_n,\g Z_n'\right\} 
	\leq \exp\left( \frac{-n\epsilon^2}{8} \right).
$$
Therefore, with $\F_{\g X_n\g X_n'} = \{(f(X_1),\dots,f(X_n), f(X_1'),\dots, f(X_n')) : f \in \F\}$, by the union bound, 
\begin{align*}
	P \left\{\left.\g D_n : \sup_{f\in\F} \frac{1}{n} \sum_{i=1}^n D_i > \frac{\epsilon}{2}  \right\vert \g Z_n,\g Z_n'\right\}
	&\leq \vert\F_{\g X_n\g X_n'}\vert \exp\left( \frac{-n\epsilon^2}{2} \right)
\end{align*}
and the conclusion stems from Lemma~\ref{lem:VCsymmetrization} and Definition~\ref{def:growthfunction}.

\section{Proof of Lemma~\ref{lem:relativesymmetrization}}
\label{sec:proofrelativesymmetrization}

The proof follows the lines of that of Lemma~2 in \cite{Cortes19} and makes use of the following. 
\begin{lemma}[Theorem 1 in \cite{Greenberg14}]
\label{lem:binomial}
For any positive integer $m$ and any probability $p$ such that $p > 1/m$, let $X$ be a random variable distributed according to the binomial distribution with $m$ trials and probability of success of each trial $p$. Then, $\E X = mp$ and 
$$
	P(X \geq \E X) > \frac{1}{4}.
$$
\end{lemma}

For any $f\in\F$, consider the events
$$
	A(f) = \left\{\frac{L_n(f)- \hat{L}_n(f)}{\sqrt{L_n(f)} } > \epsilon \right\} \quad \text{and}\quad B(f)= \{\hat{L}_n'(f) > L_n(f)\}.
$$
Then
$$
	A(f) \cap B(f) \quad \subset \quad C(f) =  \left\{\frac{\hat{L}_n'(f) - \hat{L}_n(f)}{\sqrt{(\hat{L}_n(f)+\hat{L}_n'(f) + \frac{1}{n})/2}} > \epsilon\right\}.
$$
To see this, first note that 
$$
	A(f) \quad \Rightarrow \quad \hat{L}_n(f) < L_n(f) -\epsilon\sqrt{L_n(f)},
$$
which also implies $\epsilon <\sqrt{L_n(f)}$. 
Then, also note that, for all $a, b>0$, the function $\frac{a - b}{\sqrt{(a+b + \frac{1}{n})/2}}$ is increasing in $a$ and decreasing in $b$ (the derivatives are positive and negative). Thus, 
$$
	A(f) \Rightarrow \quad \frac{\hat{L}_n'(f) - \hat{L}_n(f)}{\sqrt{(\hat{L}_n(f)+\hat{L}_n'(f) + \frac{1}{n})/2}} \geq \frac{\hat{L}_n'(f) - L_n(f) + \epsilon\sqrt{L_n(f)}}{\sqrt{(\hat{L}_n'(f) +L_n(f) - \epsilon\sqrt{L_n(f)}+ \frac{1}{n})/2}}, 
$$
where, on $B(f)$,
\begin{align*}
	 \frac{\hat{L}_n'(f) - L_n(f) + \epsilon\sqrt{L_n(f)}}{\sqrt{(\hat{L}_n'(f) +L_n(f) - \epsilon\sqrt{L_n(f)}+ \frac{1}{n})/2}} 
	 & \geq  \frac{\epsilon\sqrt{L_n(f)}}{\sqrt{(2L_n(f)  - \epsilon\sqrt{L_n(f)} + \frac{1}{n})/2}} 	\\
	  & \geq  \frac{\epsilon\sqrt{L_n(f)}}{\sqrt{(2L_n(f)  - \epsilon^2 + \frac{1}{n})/2}}
\end{align*}
(because $\epsilon <\sqrt{L_n(f)}$).
Since we assume $n\epsilon^2 > 1$ and thus $\epsilon^2 > 1/n$, we have $ - \epsilon^2 + \frac{1}{n} < 0$ and 
\begin{align*}
	 \frac{\hat{L}_n'(f) - L_n(f) + \epsilon\sqrt{L_n(f)}}{\sqrt{(\hat{L}_n'(f) +L_n(f) - \epsilon\sqrt{L_n(f)}+ \frac{1}{n})/2}} & \geq  \frac{\epsilon\sqrt{L_n(f)}}{\sqrt{L_n(f)}} = \epsilon .
\end{align*}

Now, let $f^*\in\F$ denote the function that depends solely on $\g Z_n$ and such that\footnote{In case this is not possible, choose $f^*$ as the function for which the ratio is $\eta$-close to the supremum and take the limit $\eta\to 0$ to conclude as in \cite{Cortes19}.}
$$
	  \frac{L_n(f^*) - \hat{L}_n(f^*)}{\sqrt{L_n(f)}} = \sup_{f\in\F} \frac{L_n(f) - \hat{L_n}_n(f)}{\sqrt{L_n(f)}} .
$$
Then, by the discussion above,
\begin{align*}
	P\{C(f^*)\}&\geq P\{A(f^*) \cap B(f^*)\} \\
	&=\E \I{A(f^*)}\I{B(f^*)}\\
	&= \E \E\left[\left. \I{A(f^*)}\I{B(f^*)} \right\vert\g Z_n\right]\\
	&= \E \I{A(f^*)} \E\left[\left. \I{B(f^*)} \right\vert \g Z_n\right]\\
	&= \E \I{A(f^*)} P\{B(f^*) \mid \g Z_n\}.
\end{align*}
Moreover, given $\g Z_n$, $f^*$ is fixed and if $L_n(f^*) >  \epsilon^2$, then under the assumption $n\epsilon >1$, $L_n(f) > 1/n$. Then, by stationarity and the construction of the ghost sample, the $Z_i'$ are i.i.d., which makes the quantity 
$$
	n\hat{L}_n'(f^*) = \sum_{i=1}^n \I{f^*(X_i')\neq Y_i'}
$$
distributed according to a binomial distribution with $m=n$ trials and probability of success $p=P(f^*(X_i')\neq Y_i') = P(f^*(X_i)\neq Y_i) = L_n(f^*)>1/m$. Thus, by Lemma~\ref{lem:binomial}, $P\{B(f^*) \mid \g Z_n\} > 1/4$ and, in the case $L_n(f^*) >  \epsilon^2$, $P\{C(f^*)\} \geq \E \I{A(f^*)}/4$ and $P\{A(f^*)\} \leq 4 P\{C(f^*)\}$. In the other case, if $L_n(f^*)\leq  \epsilon^2$, then $A(f^*)$ cannot hold and $P\{A(f^*)\}=0 \leq 4 P\{C(f^*)\}$. So, in any case:
\begin{align*}
	P \left\{ \sup_{f\in\F}\frac{L_n(f) - \hat{L}_n(f)}{\sqrt{L_n(f)}}  \geq \epsilon \right\}  &= P\{A(f^*)\}\\ &\leq 4 P\{C(f^*)\} \\
	&\leq 4 P \left\{ \sup_{f\in\F}\frac{\hat{L}_n'(f) - \hat{L}_n(f)}{\sqrt{(\hat{L}_n(f)+\hat{L}_n'(f) + \frac{1}{n})/2}} > \epsilon \right\},
\end{align*}
where the first equality is due to the choice of $f^*$ and the last inequality is due to $f^*\in\F$.

\section{Proof of Theorem~\ref{thm:generalVCbound}}
\label{sec:proofgeneralVCbound}

We proceed as in the proof of Theorem~\ref{thm:VCbound} (see App.~\ref{sec:proofVCbound}) with  a sequence $\g\sigma_n=(\sigma_i)_{1\leq i\leq n}$  of independent Rademacher variables uniformly distributed in $\{+1,-1\}$ and by taking advantage of the symmetry of the $(\I{f(X_i')\neq Y_i'}- \I{f(X_i)\neq Y_i} )$:
\begin{align*}
	P&\left\{ \sup_{f\in\F} \frac{\hat{L}_n'(f) - \hat{L}_n(f)}{\sqrt{(\hat{L}_n(f)+\hat{L}_n'(f) + \frac{1}{n})/2}} > \epsilon\right\} \\
	&= \E \I{}\left\{ \sup_{f\in\F} \frac{ \frac{1}{n} \sum_{i=1}^n \sigma_i(\I{f(X_i')\neq Y_i'} - \I{f(X_i)\neq Y_i} )}{\sqrt{(\hat{L}_n(f)+\hat{L}_n'(f) + \frac{1}{n})/2}} > \epsilon\right\} \\
	&= \E \ \E\left[\left.  \I{}\left\{ \sup_{f\in\F} \frac{ \frac{1}{n} \sum_{i=1}^n \sigma_i(\I{f(X_i')\neq Y_i'} - \I{f(X_i)\neq Y_i} )}{\sqrt{(\hat{L}_n(f)+\hat{L}_n'(f) + \frac{1}{n})/2}} > \epsilon\right\} \right\vert \g Z_n,\g Z_n'\right]\\
	&= \E\  P \left\{\left.\g\sigma_n : \sup_{f\in\F} \frac{ \frac{1}{n} \sum_{i=1}^n \sigma_i(\I{f(X_i')\neq Y_i'} - \I{f(X_i)\neq Y_i} )}{\sqrt{(\hat{L}_n(f)+\hat{L}_n'(f) + \frac{1}{n})/2}} > \epsilon \right\vert \g Z_n,\g Z_n'\right\} 	\\
	&= \E\  P \left\{\left.\g D_n : \sup_{f\in\F} \frac{1}{n} \sum_{i=1}^n D_i > \epsilon \right\vert \g Z_n,\g Z_n'\right\} ,
\end{align*}
where we let $\g D_n=(D_i)_{1\leq i\leq n}$ and $D_i = \frac{\sigma_i(\I{f(X_i')\neq Y_i'} - \I{f(X_i)\neq Y_i} )}{\sqrt{(\hat{L}_n(f)+\hat{L}_n'(f) + \frac{1}{n})/2}}$. Since $\E \sigma_i = 0$, for any fixed $f$ and $\g Z_n$, $\g Z_n'$, we have $\E [ D_i \mid \g Z_n,\g Z_n']= 0$. Thus, by applying Hoeffding inequality (Theorem~\ref{thm:hoeffding}) on the sequence of independent variables $\g D_n$, we obtain
$$
	P \left\{\left.\frac{1}{n} \sum_{i=1}^n D_i > \epsilon \right\vert \g Z_n,\g Z_n'\right\}
	 \leq \exp\left( \frac{- 2n^2 \epsilon^2}{\sum_{i=1}^n c_i^2} \right),
$$ 
where $c_i= b_i-a_i$ with $a_i\leq D_i \leq b_i$. In particular, we have 
\begin{align*}
	a_i=\frac{-\vert\I{f(X_i')\neq Y_i'} - \I{f(X_i)\neq Y_i}\vert }{\sqrt{(\hat{L}_n(f)+\hat{L}_n'(f) + \frac{1}{n})/2}} &= \frac{-\vert\I{f(X_i')\neq Y_i'} - \I{f(X_i)\neq Y_i}\vert}{\sqrt{ \frac{1}{n}(1+\sum_{i=1}^n \I{f(X_i)\neq Y_i} + \I{f(X_i')\neq Y_i'} )/2}} \\
	&=  \frac{-\vert\I{f(X_i')\neq Y_i'} - \I{f(X_i)\neq Y_i}\vert\sqrt{2n}}{\sqrt{1+\sum_{i=1}^n \I{f(X_i)\neq Y_i} + \I{f(X_i')\neq Y_i'} }}
\end{align*}
and, similarly:
$$
	b_i = \frac{\vert\I{f(X_i')\neq Y_i'} - \I{f(X_i)\neq Y_i}\vert\sqrt{2n}}{\sqrt{1+\sum_{i=1}^n \I{f(X_i)\neq Y_i} + \I{f(X_i')\neq Y_i'} }},
$$
which gives
\begin{align*}
	c_i^2 &=  \left(\frac{2\vert\I{f(X_i')\neq Y_i'} - \I{f(X_i)\neq Y_i}\vert\sqrt{2n}}{\sqrt{1+\sum_{i=1}^n \I{f(X_i)\neq Y_i} + \I{f(X_i')\neq Y_i'} }}\right)^2
		= \frac{8n(\I{f(X_i')\neq Y_i'} - \I{f(X_i)\neq Y_i})^2}{1+\sum_{i=1}^n \I{f(X_i)\neq Y_i} + \I{f(X_i')\neq Y_i'} }\\
		&\leq \frac{8n(\I{f(X_i')\neq Y_i'} - \I{f(X_i)\neq Y_i})^2}{\sum_{i=1}^n \I{f(X_i)\neq Y_i} + \I{f(X_i')\neq Y_i'} }
\end{align*}
and 
\begin{align*}
	\sum_{i=1}^n c_i^2 & \leq \frac{8n\sum_{i=1}^n (\I{f(X_i')\neq Y_i'} - \I{f(X_i)\neq Y_i})^2}{\sum_{i=1}^n \I{f(X_i)\neq Y_i} + \I{f(X_i')\neq Y_i'} } \leq 8n,
\end{align*}
where the last inequality is due to 
\begin{align*}
 (\I{f(X_i')\neq Y_i'} - \I{f(X_i)\neq Y_i})^2 &= \I{f(X_i')\neq Y_i'}^2 + \I{f(X_i)\neq Y_i}^2 - 2 \I{f(X_i')\neq Y_i'} \I{f(X_i)\neq Y_i} \\
 	&\leq  \I{f(X_i')\neq Y_i'} + \I{f(X_i)\neq Y_i}.
\end{align*}
Thus, for any $f\in\F$, 
$$
	P \left\{\left.\g\sigma_n :  \frac{ \frac{1}{n} \sum_{i=1}^n \sigma_i(\I{f(X_i')\neq Y_i'} - \I{f(X_i)\neq Y_i} )}{\sqrt{(\hat{L}_n(f)+\hat{L}_n'(f) + \frac{1}{n})/2}} > \epsilon \right\vert \g Z_n,\g Z_n'\right\}
	\leq \exp\left( \frac{-n\epsilon^2}{4} \right)
$$
and, with $\F_{\g X_n\g X_n'} = \{(f(X_1),\dots,f(X_n), f(X_1'),\dots, f(X_n')) : f \in \F\}$, by the union bound, 
\begin{align*}
	 P & \left\{\left.\g\sigma_n : \sup_{f\in\F} \frac{ \frac{1}{n} \sum_{i=1}^n \sigma_i(\I{f(X_i')\neq Y_i'} - \I{f(X_i)\neq Y_i} )}{\sqrt{(\hat{L}_n(f)+\hat{L}_n'(f) + \frac{1}{n})/2}} > \epsilon \right\vert \g Z_n,\g Z_n'\right\}   \\
	&\qquad\qquad\qquad\qquad\qquad\qquad\qquad\qquad\qquad\leq \vert\F_{\g X_n\g X_n'}\vert \exp\left( \frac{-n\epsilon^2}{4} \right).
\end{align*}

Gathering it all and using Lemma~\ref{lem:relativesymmetrization}:
\begin{align*}
	P \left\{ \sup_{f\in\F} \frac{L(f) - \hat{L}_n(f)}{\sqrt{L(f)}} > \epsilon \right\} 
	&\leq 4 P\left\{ \sup_{f\in\F} \frac{\hat{L}_n'(f) - \hat{L}_n(f)}{\sqrt{(\hat{L}_n(f)+\hat{L}_n'(f) + \frac{1}{n})/2}} > \epsilon\right\} \\
	&\leq 4 \E  \vert\F_{\g X_n\g X_n'}\vert \exp\left( \frac{-n\epsilon^2}{4} \right)\\
	&\leq 4 \Pi_{\F}(2n) \exp\left( \frac{-n\epsilon^2}{4} \right).
\end{align*}
By setting $\delta =  4\Pi_{\F}(2n)  \exp(-n\epsilon^2/4)$, we obtain that, with probability at least $1-\delta$,
$$
	\forall f\in\F,\quad \frac{L(f) - \hat{L}_n(f)}{\sqrt{L(f)}} \leq 2\sqrt{\frac{\log \Pi_{\F}(2n) + \log\frac{4}{\delta}}{n}}
$$
and, using $A\leq B\sqrt{A} + C \Rightarrow A\leq B^2 + B\sqrt{C} + C$ yields
\begin{align*}
	& L(f) \leq \hat{L}_n(f) + \sqrt{L(f)} 2\sqrt{\frac{\log\Pi_{\F}(2n) + \log\frac{4}{\delta}}{n}}\\
	\Rightarrow\quad & L(f) \leq \hat{L}_n(f) + 2\sqrt{ \hat{L}_n(f)\frac{\log \Pi_{\F}(2n) + \log\frac{4}{\delta}}{n}} + 4\frac{\log \Pi_{\F}(2n) + \log\frac{4}{\delta}}{n},
\end{align*}
in which the growth function $\Pi_{\F}(2n)$ can be bounded in terms of the VC-dimension $d_{VC}$ by Sauer's lemma \citep{Sauer72,Vapnik98}.

\section{Proof of Theorem~\ref{thm:radbound}}
\label{sec:proofradsymmetrization}

We will show that, with probability at least $1-\delta$,
$$
	\sup_{f\in\F} \left( L_n(f) - \hat{L}_n(f) \right)
	\leq \mathcal{R}_{\g Z_n}(\L) +\mathcal{R}_{\g Z'_n}(\L)+ B\sqrt{\frac{ \log \frac{1}{\delta}}{2n}}.
$$

Let us first rewrite, with a slight abuse of notation, the loss as $\ell(Z_i) = \ell(f(X_i),Y_i)$. 
Then, we  apply a bounded difference inequality to 
$$
	g(Z_1,\dots,Z_n) = \sup_{f\in\F} \left( L_n(f) - \hat{L}_n(f)\right).
$$
If $\ell(Z_i)\in[0,B]$, then the bounded difference condition with $c_i=B/n$ is satisfied and we can apply Theorem~\ref{thm:boundeddiff}. By setting $\delta = \exp(-2t^2/\sum_{i=1}^n c_i^2) = \exp(-2n t^2/B^2)$, and thus $t=B\sqrt{\log (1/\delta) / 2n}$, we get  that, with probability at least $1-\delta$,
\begin{equation}\label{eq:radsymproof1}
	\sup_{f\in\F} \left( L_n(f) - \hat{L}_n(f) \right)
	\leq \E \sup_{f\in\F} \left( L_n(f) - \hat{L}_n(f) \right) + B\sqrt{\frac{ \log \frac{1}{\delta}}{2n}} .
\end{equation}
Then, introduce the ghost sample $\g Z_n'$ and use its independence with $\g Z_n$ and $Z_i'\sim Z_i$ to write
$$
	\E  \left[\left. \frac{1}{n} \sum_{i=1}^n \ell(Z'_i) \right\vert \g Z_n \right]=\E  \frac{1}{n} \sum_{i=1}^n \ell(Z'_i) = \frac{1}{n} \sum_{i=1}^n\E \ell(Z_i) = L_n(f),
$$
and thus
\begin{align*}
	\E \sup_{f\in\F} \left( L_n(f) - \hat{L}_n(f)  \right) 
		&=\E \sup_{f\in\F} \left( \E  \left[\left. \frac{1}{n} \sum_{i=1}^n \ell(Z'_i) \right\vert \g Z_n \right]  - \frac{1}{n} \sum_{i=1}^n \ell(Z_i) \right) \\
		&=\E \sup_{f\in\F} \left( \E  \left[\left. \frac{1}{n} \sum_{i=1}^n \ell(Z'_i) - \frac{1}{n} \sum_{i=1}^n \ell(Z_i) \right\vert \g Z_n \right]   \right) \\
		&\leq \E \E  \left[\left. \sup_{f\in\F}  \frac{1}{n} \sum_{i=1}^n \left(\ell(Z'_i)- \ell(Z_i)\right)  \right\vert \g Z_n \right] \\
		&= \E   \sup_{f\in\F}  \frac{1}{n} \sum_{i=1}^n \left(\ell(Z'_i) - \ell(Z_i)\right) ,
\end{align*}
where we used Jensen's inequality in the third line. 

Then, we can introduce Rademacher variables as follows. As in the proof of Theorem~\ref{thm:VCbound}, the variables $\left(\ell(Z'_i) - \ell(Z_i)\right)$ are symmetric by construction of the ghost sample that ensures that $Z_i \sim Z_i'$ and $Z_i$ is independent of $Z_i'$. Thus, for any $\g \sigma_n=(\sigma_i)_{1\leq i\leq n}\in\{-1,+1\}^n$ chosen independently of $Z_i$ and $Z_i'$, we can replace them by $\sigma_i\left(\ell(Z'_i) - \ell(Z_i)\right)$ without changing the resulting expectation. Furthermore, averaging over all the $2^n$ sequences of $\sigma_i$ does not change the result either. This leads to 
\begin{align*}
	 \E   \sup_{f\in\F}  \frac{1}{n} \sum_{i=1}^n \left(\ell(Z'_i) - \ell(Z_i)\right) 
	 &=  \E \sup_{f\in\F}  \frac{1}{n} \sum_{i=1}^n \sigma_i\left(\ell(Z'_i) - \ell(Z_i)\right) \\
	 &\leq  \E  \sup_{f\in\F}  \frac{1}{n} \sum_{i=1}^n \sigma_i \ell(Z'_i) +\E   \sup_{f\in\F}  \frac{1}{n} \sum_{i=1}^n -\sigma_i \ell(Z_i)\\
	 &= \E   \sup_{f\in\F}  \frac{1}{n} \sum_{i=1}^n \sigma_i \ell(Z'_i) +\E  \sup_{f\in\F}  \frac{1}{n} \sum_{i=1}^n \sigma_i \ell(Z_i) \\
	 &= \mathcal{R}_n(\L) +\mathcal{R}'_n(\L),
\end{align*}
where we used $\sigma_i\sim -\sigma_i$ in the third line. Together with~\eqref{eq:radsymproof1}, this completes the proof.

\section{Proof of~\eqref{eq:radreglin}}
\label{sec:radreglinyield}

For any $\g z_n\in\Z^n$, by the contraction principle \citep{Ledoux91}, we have $\hat{\mathcal{R}}_{\g z_n}(\L) \leq 4M\hat{\mathcal{R}}_{\g x_n}(\F)$. Then, the computations of \cite{Bartlett02} give 
$$
	\hat{\mathcal{R}}_{\g x_n}(\F)\leq \frac{\Lambda \sqrt{\sum_{i=1}^n \|X_i\|^2}}{n} \leq \frac{ \Lambda \sup_{x\in\X} \|x\|}{\sqrt{n}}
$$
for linear models. For kernel models, the norms $\|X_i\|$ are replaced by $\|K(X_i,\cdot)\| =  \sqrt{K(X_i,X_i)}$, which is $1$ for the Gaussian kernel.

\end{document}